\documentclass[conference]{IEEEtran}
\usepackage{times}

\usepackage[numbers]{natbib}
\usepackage{multicol}
\usepackage[bookmarks=true]{hyperref}

\usepackage{algorithm2e}
\usepackage{mathtools}
\usepackage{amsthm} 
\usepackage[utf8]{inputenc}
\usepackage[english]{babel}
\usepackage{graphicx}
\usepackage{subcaption}
\usepackage{color, colortbl}
\usepackage[dvipsnames]{xcolor}

\DeclarePairedDelimiter{\norm}{\lVert}{\rVert}

\theoremstyle{definition}
\newtheorem{definition}{Definition}[]
\newtheorem{theorem}{Theorem}
\definecolor{Gray}{gray}{0.9}
\definecolor{Red}{RGB}{227,74,51}
\definecolor{Green}{RGB}{127,205,187}
\definecolor{Blue}{RGB}{44,127,184}
\definecolor{GrayPlot}{RGB}{240,240,240}

\begin{document}

% paper title
\title{A Differentiable Augmented Lagrangian Method for Bilevel Nonlinear Optimization}

% You will get a Paper-ID when submitting a pdf file to the conference system
% \author{Author Names Omitted for Anonymous Review. Paper-ID 63}

\author{\IEEEauthorblockN{Benoit Landry, Zachary Manchester and Marco Pavone}
\IEEEauthorblockA{Department of Aeronautics and Astronautics\\
Stanford University\\
Stanford, CA, 94305\\
Email: \{blandry,zacmanchester,pavone\}@stanford.edu}}

% avoiding spaces at the end of the author lines is not a problem with
% conference papers because we don't use \thanks or \IEEEmembership

\maketitle

\begin{abstract}
Many problems in modern robotics can be addressed by modeling them as bilevel optimization problems. In this work, we leverage augmented Lagrangian methods and recent advances in automatic differentiation to develop a general-purpose nonlinear optimization solver that is well suited to bilevel optimization. We then demonstrate the validity and scalability of our algorithm with two representative robotic problems, namely robust control and parameter estimation for a system involving contact. We stress the general nature of the algorithm and its potential relevance to many other problems in robotics.
\end{abstract}

\IEEEpeerreviewmaketitle

\section{Introduction}
Bilevel optimization is a general class of optimization problems where a lower level optimization is embedded within an upper level optimization. These problems provide a useful framework for solving problems in modern robotics that have a naturally nested structure such as some problems in motion planning \cite{CariusRanftlEtAl2018,ToussaintAllenEtAl2018}, estimation \cite{AvilaBelbute-PeresSmithEtAl2018} and learning \cite{AmosKolter2017,FinnAbbeelEtAl2017}. However, despite their expressiveness, bilevel optimization problems remain difficult to solve in practice, which currently limits their impact on robotics applications. 

To overcome these difficulties, this work leverages recent advances in the field of automatic differentiation to develop a nonlinear optimization algorithm based on augmented Lagrangian methods that is well suited to automatic differentiation. The ability to differentiate our solver allows us to combine it with a second, state-of-the-art nonlinear program solver, such as SNOPT \cite{GillMurrayEtAl2005} or Ipopt \cite{WachterBiegler2006} to provide a solution method to bilevel optimization problems.

\begin{figure}
    \centering
    \includegraphics[trim={9cm 8cm 12cm 8cm},clip,width=\columnwidth]{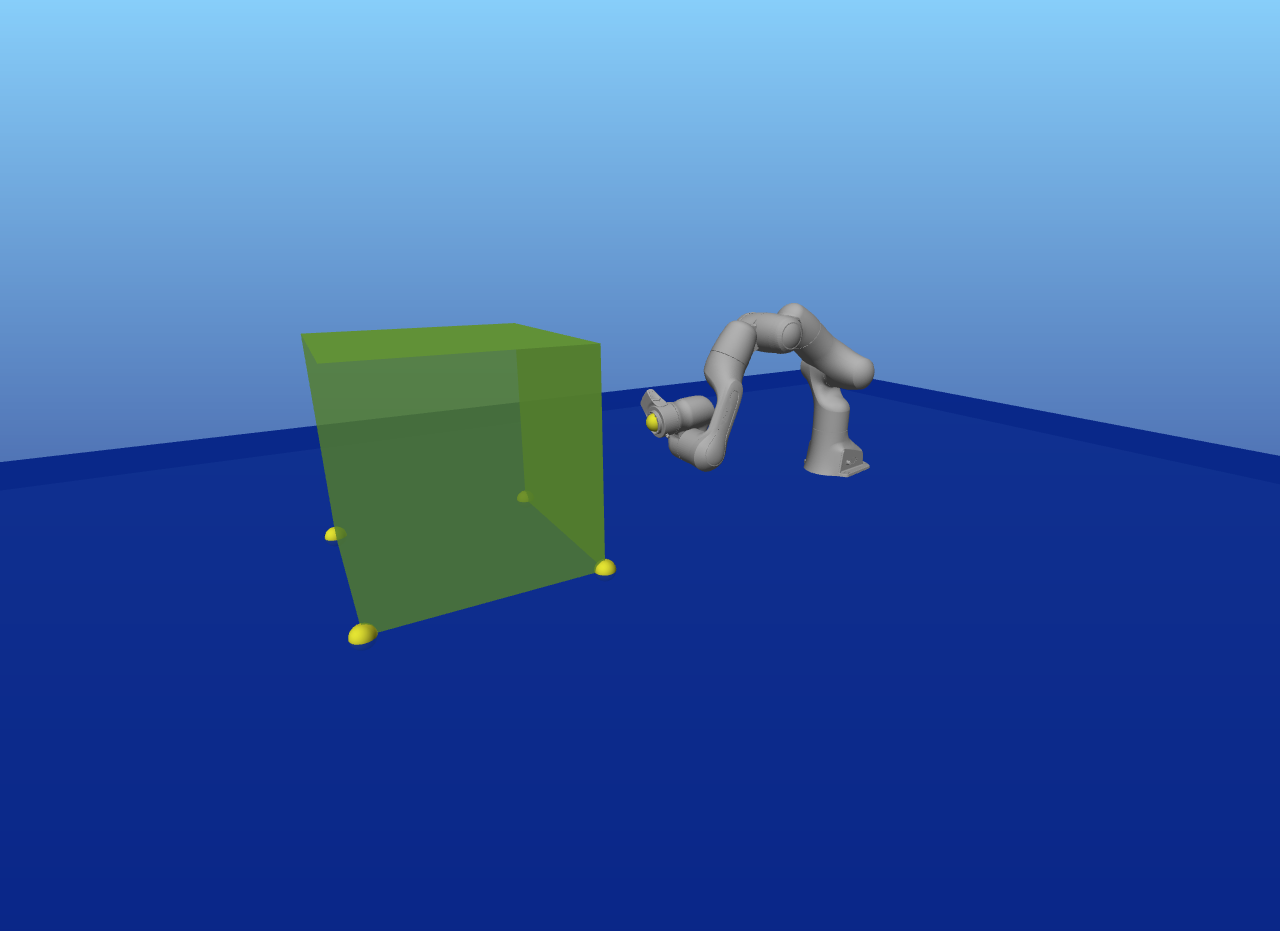}
    \caption{For our parameter estimation example, a robotic arm pushes a box (simulated with five contact points - the yellow spheres - and Coulomb friction) and then estimates the coefficient of friction between the box and the ground. This is done by solving a bilevel program using a combination of SNOPT and our differentiable solver. Unlike alternative formulations, the bilevel formulation of this estimation problem scales well with the number of samples used because it can easily leverage parallelization of the lower problems.}
    \label{fig:push}
\end{figure}

Similar solution methods can be found in various forms across many disciplines. However, our approach differentiates itself from previous methods in a few respects. First, we keep the differentiable solver general in that it is a nonlinear optimization algorithm capable of handling both equality and inequality constraints without relying on projection steps. Our method also does not depend on using the Karush–Kuhn–Tucker (KKT) conditions in order to retrieve the relevant gradients, as described in various work going as far back as \cite{LevyRockafellar1995}, and more recently in the context of machine learning \cite{AmosKolter2017}. Instead, closer to what was done for a specialized physics engine in \cite{DegraveHermansEtAl2017}, we carefully implement our solver without any branching or non-differentiable operations and auto-differentiate through it. We therefore pay a higher price in computational complexity but avoid some of the limitations of relying on the KKT conditions (like having to solve the lower problem to optimality for example). We also leverage recent advances in automatic differentiation by implementing our solver in the Julia programming language, which is particularly well suited for this task. When applying the method to robotics, unlike a lot of previous work that combines differentiable solvers with stochastic gradient descent in a machine learning context, we demonstrate how to incorporate our solver with another state-of-the-art solver that uses sequential quadratic programming (SQP) to address problems relevant to robotics. Among other things, this allows us to use our solver to handle constraints in bilevel optimization problems, not just unconstrained objective functions \cite{AvilaBelbute-PeresSmithEtAl2018,DegraveHermansEtAl2017}.

{\em Statement of Contributions:} Our contributions are two fold. First, by combining techniques from the augmented Lagrangian literature and leveraging recent advances in automatic differentiation, we develop a nonlinear optimization algorithm that is entirely differentiable. Second, we demonstrate the validity and scalability of the approach on two important problems in modern robotics: robust control and parameter estimation for systems with complex dynamics. All of our code is available at \url{https://github.com/blandry/Bilevel.jl}.

\section{Related Work}

\subsection{Bilevel Optimization}
Bilevel optimization problems are mathematical programs that contain the solution of other mathematical programs in their constraints or objectives. In this work we refer to the primary optimization problem as the \textit{upper} problem, and the embedded optimization problem as the \textit{lower} problem. A more detailed definition of the canonical bilevel optimization problem can be found in Section \ref{sec:bilevel}. 

There have been many proposed solutions to general bilevel optimization problems. When the lower problem is convex and regular (with finite derivative), one approach is to directly include the KKT conditions of the lower problem in the constraints of the upper problem. However except for special cases, the complementarity aspect of the KKT conditions can yield combinatorial optimization problems that require enumeration-based solution methods such as branch-and-bound, which can be nontrivial to solve \cite{BardFalk1982,Fortuny-AmatMcCarl1981,BialasKarwan1984,ChenFlorian1992,TuyMigdalasEtAl1993}.
Other proposed solutions to bilevel optimization problems use \emph{descent methods}. These methods express the lower problem solution as a function of the upper problem variables and try to find a descent direction that keeps the bilevel program feasible. These descent directions can be estimated in various ways, such as by solving an additional optimization problem \cite{KolstadLasdon1990,SavardGauvin1994}. Descent methods contain some important similarities to our proposed method, but are usually more limiting in the types of constraints they can handle.
Finally, in recent years many derivative free methods for bilevel optimization (sometimes called \emph{evolutionary} methods) have gained popularity. We refer readers to \cite{ColsonMarcotteEtAl2007,SinhaMaloEtAl2018} for a more detailed overview of solution methods to bilevel optimization problems.

\subsection{Differentiating Through Mathematical Programs}
Recently there has been a resurgence of interest in the potential for differentiating through optimization problems due to the success of machine learning. One of the more successful approaches to date solves the bilevel program's lower problem using an interior point method and computes the gradient \emph{implicitly}  \cite{AmosKolter2017}. The computation of the gradient leverages results in sensitivity analysis of nonlinear programming \cite{LevyRockafellar1995,FiaccoIshizuka1990,RalphDempe1995}.

Closest to our proposed method, \emph{unrolling} methods attempt to differentiate through various solution algorithms directly \cite{Domke2012,AmosXuEtAl2017,BelangerYangEtAl2017}. To the best of our knowledge, unrolling approaches are either limited to unconstrained lower optimization problems (which include meta-learning approaches like \cite{FinnAbbeelEtAl2017}), or the lower problem solver is a special purpose solver that handles its constraints using projection steps, which is similar to what is done in \cite{CariusRanftlEtAl2018} or \cite{DegraveHermansEtAl2017}.

\subsection{Robotics Applications of Bilevel Optimization}
Several groups have explored problems in robotics using bilevel programming solution methods. For example \cite{CariusRanftlEtAl2018} explores a trajectory optimization problem for a hopping robot. In their work the lower problem consists of rigid body dynamics with contact constraints that uses a special purpose solver (that relies on a projection step) and the upper problem is the trajectory optimization which is solved by dynamic programming. This work is related to our example described in Section \ref{subsec:robust_result}, where the upper problem in the bilevel program is also a trajectory optimization problem. Trajectory generation is also addressed in \cite{ToussaintAllenEtAl2018}, which explores the use of bilevel optimization in the context of integrated task and motion planning (with a mixed-integer upper problem).

In the machine learning community, there have been attempts at including a differentiable solver inside the learning process (e.g. in an unconstrained minimization method such as stochastic gradient descent). Similar to our parameter estimation example described in Section \ref{subsec:param}, the work in \cite{AvilaBelbute-PeresSmithEtAl2018} leverages the approach in \cite{AmosKolter2017} for a learning task that includes a linear complementarity optimization problem (arising from dynamics with hard contact constraints) as a lower problem. This work then uses gradient descent to solve upper problems related to parameter estimation. Our approach to the parameter estimation problem differentiates itself from \cite{AvilaBelbute-PeresSmithEtAl2018} by using a different solution method for both the lower and upper problems. Specifically, we solve the lower problem using our proposed augmented Lagrangian method and solve the upper one using a sequential quadratic program (SQP) solver. 
This approach allows us to handle constraints on the estimated parameters, which is not possible using gradient descent like in \cite{AvilaBelbute-PeresSmithEtAl2018}. Our work also differs in the fact that we use full nonlinear dynamics (instead of linearized dynamics) and demonstrate our approach on a 3D example (instead of 2D ones). The authors of \cite{DegraveHermansEtAl2017} also develop a differentiable rigid body simulator capable of handling contact dynamics by developing a special purpose solver and using it in conjunction with gradient descent for various applications.

\section{Bilevel Optimization}\label{sec:bilevel}
The ability to differentiate through a nonlinear optimization solver allows us to address problems in the field of bilevel programming. Here, we carefully define the notation related to bilevel optimization.

In bilevel programs, an \emph{upper} optimization problem is formulated with the result of a \emph{lower} optimization problem appearing as part of the upper problem's objective, constraints, or both.
In this work we formulate the bilevel optimization problem using the following definitions, from \cite{SinhaMaloEtAl2018}:
% Reusing some of the definitions in \cite{sinha2018review}, we can formulate bilevel optimization problems using Definitions \ref{def:bilevel} and \ref{def:lower}.

\theoremstyle{definition}
\begin{definition}{Bilevel Optimization}\label{def:bilevel}\\
A bilevel program is a mathematical program (or optimization problem) that can be written in the form:
\begin{equation}
\begin{aligned}
& \underset{x_u \in X_U, x_l \in X_L}{\text{minimize}}
& & F(x_u,x_l) \\
& \text{subject to}
& & x_l \in \underset{x_l \in X_L}{\text{argmin}}\{f(x_u,x_l): \\
&
& & \hspace{20mm} g_i(x_u,x_l) \leq 0, i = 1, \ldots, m \\
& 
& & \hspace{20mm} h_j(x_u,x_l) = 0, j = 1, \ldots, n \}, \\
& 
& & G_k(x_u,x_l) \leq 0, k = 1, \ldots, M, \\
&
& & H_m(x_u,x_l) = 0, m = 1, \ldots, N,
\end{aligned}\label{eq:bilevel_prob}
\end{equation}
where $x_u$ represents the decision variables of the upper problem and $x_l$ represents the decision variables of the lower problem. The nonlinear expressions $f$, $g$ and $h$ represent the objective and constraints of the lower problem, and the nonlinear expressions $F$, $G$ and $H$ represent the objective and constraints of the upper problem. Notably, the $\text{argmin}$ operation in problem \ref{eq:bilevel_prob} returns a \emph{set} of optima for the lower problem.
\end{definition}

\theoremstyle{definition}
\begin{definition}{Lower Problem Solution}\label{def:lower}\\
We define the set-valued function $\Psi(x_u)$ as the solution to the following optimization problem:
\begin{equation}
\begin{aligned}
& \Psi(x_u) = 
& & \underset{x_l \in X_L}{\text{argmin}}\{ f(x_u,x_l): \\
& 
& & \hspace{12mm} g_i(x_u,x_l) \leq 0, i = 1, \ldots, I,\\
& 
& & \hspace{12mm} h_j(x_u,x_l) = 0, j = 1, \ldots, J \},
\end{aligned}
\end{equation}
where the definition of $x_u$, $x_l$, $f$, $g$ and $h$ are the same as in Definition \ref{def:bilevel}.
\end{definition}

\section{A Differentiable Augmented Lagrangian Method}

In order to solve bilevel optimization problems, we propose to implement a simple solver based on augmented Lagrangian methods. The solver we present carefully combines techniques from the nonlinear optimization community in order to achieve a set of properties that are key in making our solver useful in the context of bilevel optimization. In our proposed solution to bilevel optimization, an upper solver (like an off-the-shelf SQP solver) makes repeated calls to a lower solver in order to obtain the optimal value of the lower problem and the corresponding gradients at that solution. These gradients are the partials of the optimal solution with respect to each parameter of the lower problem, which themselves correspond to the variables of the upper problem. To achieve this, the key properties that the lower solver must have include differentiability, robustness and efficiency.

\subsection{Solver Overview}\label{subsec:algooverview}
We now give a quick overview of our solver. First, the solver converts inequality constraints to equality constraints by introducing extra decision variables for reasons discussed in Section \ref{subsec:diff}. It then uses the robust first-order augmented Lagrangian method to warm-start a full primal-dual Newton method for fast tail convergence. Those initial first-order updates notably discover the constraints active set and actually use a second-order update on the primal (but not dual) variables. The differentiability of these update steps is preserved through design choices again explained in Section \ref{subsec:diff}. The subsequent primal-dual Newton method on the augmented Lagrangian is closely related to SQP. The combination of a first-order method with a second-order one is justified in section \ref{subsec:robust}. Algorithm \ref{algo:auglag} contains an overview of the solver just described, with $\norm{\cdot}_F$ representing the Frobenius norm.

\begin{algorithm}
\SetAlgoLined
\KwResult{$x$, $\lambda$ s.t. $x \approx \text{argmin}f(x); h_0(x) \approx 0; g_0(x) \leq \epsilon$} 
 $h \leftarrow \{h_0 \cup \text{convert}(g_0)\}$ \\
 $x,\lambda \leftarrow 0$ \\
 \While{$i \leq \text{num\ first\ order\ iterations}$}{
    $g \leftarrow \nabla{f}(x) - \nabla{h}(x)^T\lambda + c\nabla{h}(x)^Th(x)$ \\
    \vspace{1mm}
    $\mathbf{H} \leftarrow \nabla_{2}{f}(x) + c\nabla{h(x)}^T\nabla{h(x)}$  \\
    \vspace{1mm}
    $x \leftarrow x - (\mathbf{H} + (\norm{\mathbf{H}}_F + \epsilon) \cdot \mathbf{I})^{-1} g$ \\
    \vspace{1mm}
    $\lambda \leftarrow \lambda - c \cdot h(x)$\\
    $c \leftarrow \alpha \cdot c$\\
    $i \leftarrow i + 1$
    }
 \While{$j \leq \text{num\ second\ order\ iterations}$}{
    $g \leftarrow \nabla{f}(x) - \nabla{h}(x)^T\lambda + c\nabla{h}(x)^Th(x)$ \\
    $\mathbf{H} \leftarrow \nabla_{2}{f}(x) + c\nabla{h(x)}^T\nabla{h(x)}$ \\
    \vspace{2mm}
    $\mathbf{K} \leftarrow 
    \begin{bmatrix}
        \mathbf{H} & \nabla{h}^T \\
        \nabla{h} & 0
    \end{bmatrix} $\\
    \vspace{2mm}
    $\mathbf{U},\mathbf{S},\mathbf{V} \leftarrow \text{svd}(\mathbf{K})$ \\
    \vspace{2mm}
    $\mathbf{S}_i \leftarrow \Gamma(\mathbf{S}) \mathbf{S}^{-1} $\\
    \vspace{2mm}
    $\begin{bmatrix}x \\ \lambda \end{bmatrix} \leftarrow \begin{bmatrix}x \\ \lambda \end{bmatrix} - \mathbf{V} \mathbf{S}_i \mathbf{U}^T \begin{bmatrix} g \\ h(x) \end{bmatrix} $\\
    \vspace{2mm}
    $j \leftarrow j + 1$
 }
 \caption{Differentiable Augmented Lagrangian Method}
 \label{algo:auglag}
\end{algorithm}

\subsection{Differentiability}\label{subsec:diff}
In order to keep the solver differentiable, it is important to avoid any branching in the algorithm. In the context of nonlinear optimization this can turn out to be difficult. Indeed, most nonlinear optimization methods include a line-search component which would be difficult to implement without branching. In order to get around this problem, we strictly use second-order updates on the primal as described above. This is more computationally demanding than first-order updates, but it gives a useful step-size estimate without requiring a line search. Another complication brought on by using second-order updates is that the Hessian of the augmented Lagrangian function used to compute the second-order updates needs to be sufficiently positive definite in order to be useful. This problem is usually handled by either adding a multiple of identity to the Hessian such that all its eigenvalues are above some numerical tolerance $\delta$, or by projecting the matrix to the space of positive definite matrices using a singular value decomposition \cite{NocedalWright2006}. When using the first method, the diagonal correction matrix $\Delta H$ with the smallest euclidean norm that satisfies $\lambda_{\min}(H + \Delta H) \geq \delta$ for some numerical tolerance $\delta$ is given by
\begin{equation}
\Delta H^* = \tau I,\text{ with }\tau = \max(0, \delta - \lambda_{\min}(H)).
\end{equation}

Since our first few optimization steps on the primal variables are only used to compute a good initial guess for the dual variables, the matrix $\Delta H$ does not need to be exactly optimal (in terms of keeping the steepest descent direction intact). We therefore choose to instead solve for $\Delta H$ in a fast and easily differentiable way by using $\Delta H = (\norm{H}_F + \delta)I$. We prove that this modification makes the Hessian sufficiently positive definite in Theorem \ref{theo:matrixmod}.

\theoremstyle{theorem}
\begin{theorem}[Conservative Estimate of Required Matrix Modification]
\label{theo:matrixmod}
% 1) first we know that adding a multiple of the identity to a matrix increases its eigenvalues by that ammount
% 2) now we know that for a square symetric matrix, the largest singular value equals the largest eigen value in the absolute value
% 3) we know that the frobenius norm is an upper bound to the largest singular value
% 4) so adding F+d makes every eigenvalues of A at least d
Let $A$ be a square, symmetric matrix and $\delta$ a non-negative scalar, then $\Delta A := (\norm{A}_F + \delta)I$ is a matrix such that $\lambda_{\min}(A + \Delta A) \geq \delta$, where $\lambda_{\min}$ is the smallest eigenvalue of $A$.
\end{theorem}
\begin{proof}
First, we know that adding a scalar multiple of the identity to a square matrix increases its eigenvalues by precisely that scalar amount since
\begin{equation}
    Av = \lambda v \Rightarrow (A + \delta I)v = (\lambda + \delta) v.
\end{equation}
We also know that for a square Hermitian matrix, the largest singular value is equal to the largest eigenvalue in absolute value i.e. $|\lambda(A)|_{\max} = \sigma_{\max}$. Next, since multiplication by orthonormal matrices leaves the Frobenius norm unchanged, we can show that the Frobenius norm of the matrix $A$ is an upper bound on its smallest singular value.
\begin{equation}
\begin{split}
\norm{A}_F & = \norm{U \Sigma V^T}_F \\
& = \norm{\Sigma}_F \\
& = \sqrt{\sigma_1^2+\ldots+\sigma_n^2} \\
& \geq \sigma_{\max}(A) \\
& = |\lambda(A)|_{\max} \\
& \geq |\lambda_{\min}(A)|
\end{split}
\end{equation}
where $U \Sigma V^T$ is the singular-value decomposition of $A$. We can now conclude that the eigenvalues of $A + (\norm{A}_F + \delta)I$ are at least as large as $\delta$ i.e.
\begin{equation}
    \lambda_{\min}(A + (\norm{A}_F + \delta)I) \geq \delta.
\end{equation}
\end{proof}

In the second part of our algorithm we set up the KKT system corresponding to the augmented Lagrangian with inequalities treated as equalities. This is effectively sequential quadratic programming. A key challenge here is that the resulting system can end up being singular, and different solvers handle this in various ways. In our proposed algorithm, we use the Moore-Penrose pseudoinverse to handle the potential singularity. The computation of this pseudoinverse relies on a singular value decomposition (SVD), which has a known derivative \cite{PapadopouloLourakis2000}, and therefore differentiability of the overall algorithm can be maintained. Note that for numerical stability, implementations of the pseudoinverse computation must usually set singular values below some threshold to zero. To keep this important operation continuous, and therefore retain differentiability, we perform this correction using a shifted sigmoid function (denoted as $\Gamma$ in Algorithm \ref{algo:auglag}).

Note that repeated singular values can present a challenge when computing the gradient of singular value decomposition. Even though there exists methods to precisely define such gradients \cite{PapadopouloLourakis2000}, in practice we found that one of many ways of preventing to gradients from getting too large (like zeroing them) worked well enough for our applications. It is also possible to define the needed gradient using the solution of the least-squares problem the pseudoinverse is used to solve, but we leave this investigation to future work.

Since SVD is computationally expensive, its use is currently the main bottleneck of our algorithm. However there exist other methods of computing the pseudoinverse, for example, when a good initial guess is available \cite{Ben-IsraelCohen1966}, that offer promising directions to improve on this.

A significant source of non-smoothness is often related to inequality constraints. Here, we get around this problem by introducing shaped slack variables and converting the inequality constraints to equality ones, leveraging the equivalence
\begin{equation}
g(x) \leq 0 \Leftrightarrow \exists y \text{ s.t. } g(x) + \xi(y) = 0,
\end{equation}

where $\xi(y)$ is the \emph{non-normalized} softmax operation

\begin{equation}
\xi(y) = \frac{\log(e^{ky} + 1)}{k},
\end{equation}
and $k$ is some parameter that adjusts the stiffness of the operation.

\subsection{Robustness}\label{subsec:robust}
An important aspect of our algorithm is that it must be able to solve nonlinear optimization problems with potentially bad initial guesses on their solution. We handle this challenge by combining first-order method-of-multipliers steps with primal-dual second-order ones as described in Section  \ref{subsec:algooverview}. Indeed, first-order method of multipliers based on minimization of the augmented Lagrangian have been shown to have slow convergence but to be more robust to bad initial guesses \cite{Bertsekas1996}. On the other hand, second order methods have provable quadratic convergence when close to the solution. Our algorithm therefore starts by performing a few update steps of the first order method followed by  additional steps of a second order one.

\subsection{Efficiency}
An important aspect of our algorithm is that it must be efficient enough to be called multiple times by an upper solver when solving bilevel optimization problems. We therefore leverage recent advances in the field of automatic differentiation by implementing our algorithm with state of the art tools: namely the ForwardDiff.jl library \cite{RevelsLubinEtAl} and the Julia programming language \cite{BezansonKarpinskiEtAl2012}.

\section{Validity of the Algorithm on a Toy Problem with Known Solution}\label{sec:toyprob}

In order to show that our approach is capable of recovering the correct solution for various bilevel optimization problems, we compare the solution returned by our solver with the known closed-form solution of a simple problem.

The operations research literature is full of problems that can be cast in the bilevel optimization framework. The game theory literature also contains many such problems, specifically in the context of Stackelberg competitions. Here we take a simple Stackelberg competition problem from \cite{SinhaMaloEtAl2018}. In this problem, two companies are trying to maximize their profit by producing the right amount of a product that is also being sold by a competitor, and that has a value inversely proportional to its availability on a shared market.

More specifically, we are interested in optimizing the production quantity for the \emph{leader}, denoted as $q_l$, knowing that the competitor (the \emph{follower}) will respond by optimizing its own production quantity $q_f$. The price of the product on the market is dictated by a linear relationship $P(q_l,q_f)$, and the production costs for each company follow quadratic relationships $C_l(q_l)$ and $C_f(q_f)$. We therefore have the following functions
\begin{equation}
\begin{aligned}
P(q_l,q_f) = \alpha - \beta (q_l + q_f) \\
C_l(q_l) = \delta_l q_l^2 + \gamma_l q_l + c_l\\
C_f(q_f) = \delta_f q_f^2 + \gamma_f q_f + c_f,
\end{aligned}
\end{equation}
where $\alpha$, $\beta$, $\delta_l$, $\delta_f$, $\gamma_l$, $\gamma_f$, $c_l$ and $c_f$ are fixed scalars for a given instantiation of the problem. We can now write the bilevel optimization problem as
\begin{equation}
\begin{aligned}
& \underset{q_l,q_f}{\text{maximize}}
& & P(q_l,q_f)q_l - C_l(ql) \\
& \text{subject to}
& & q_f \in \underset{q_f}{\text{argmax}}\{ P(q_l,q_f)q_f - C_f(q_f) \} \\
& 
& & q_l, q_f \geq 0.
\end{aligned}
\end{equation}

For given parameters, the closed-form solution of the optimal production rate of leader is given by the expression
\begin{equation}
\begin{aligned}
q_l^* = \frac{2(\beta + \delta_f)(\alpha - \gamma_l) - \beta(\alpha - \gamma_f)}{4(\beta + \delta_f)(\beta + \delta_l) - 2 \beta^2}.
\end{aligned}
\end{equation}

In order to show that our solver can help recover the closed-form solution numerically, we define the following lower problem
\begin{equation}
\begin{aligned}
\Psi(q_l) = \underset{q_f}{\text{argmax}}\{ P(q_l,q_f)q_f - C_f(q_f): q_f \geq 0\},\\
\end{aligned}
\end{equation}
which leads to the following bilevel optimization
\begin{equation}
\begin{aligned}
& \underset{q_l}{\text{maximize}}
& & P(q_l,\psi(q_f))q_l - C_l(ql) \\
& \text{subject to}
& & q_l \geq 0.
\end{aligned}
\end{equation}

We then solve the lower problem with our proposed solver and the upper one by using SNOPT \cite{GillMurrayEtAl2005}, a state-of-the-art SQP solver. The recovered optimal value for the leader's production amount is shown in figure \ref{fig:closedform}. Even though this is a very simple problem to solve, it is clear that our proposed solver successfully recovers its solution, which is a valuable sanity check of our approach.

\begin{figure}
    \centering
    \begin{subfigure}[b]{0.49\columnwidth}
        \includegraphics[width=\linewidth]{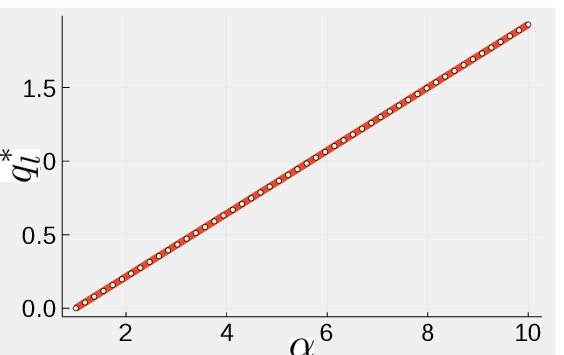}
    \end{subfigure}
    \hfill
    \begin{subfigure}[b]{0.49\columnwidth}
        \includegraphics[width=\linewidth]{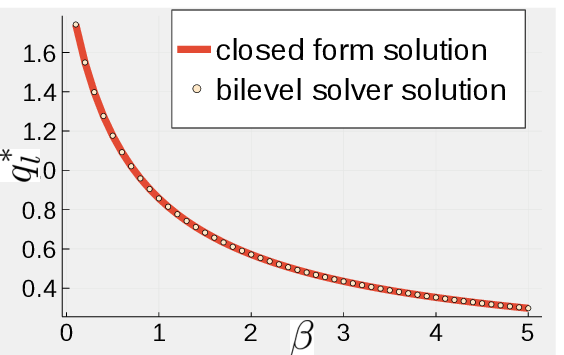}
    \end{subfigure}
    \hfill
    \begin{subfigure}[b]{0.49\columnwidth}
        \includegraphics[width=\linewidth]{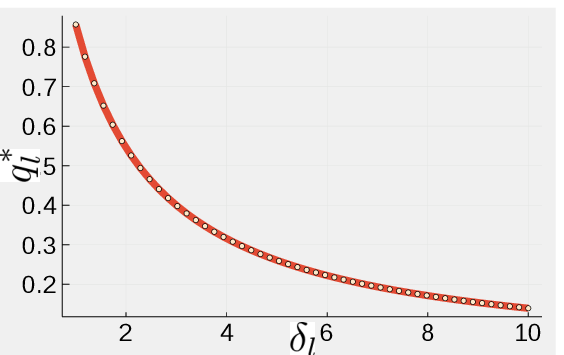}
    \end{subfigure}
    \hfill
    \begin{subfigure}[b]{0.49\columnwidth}
        \includegraphics[width=\linewidth]{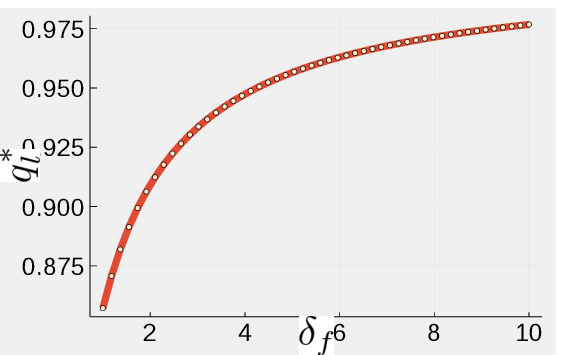}
    \end{subfigure}
    \hfill
    \begin{subfigure}[b]{0.49\columnwidth}
        \includegraphics[width=\linewidth]{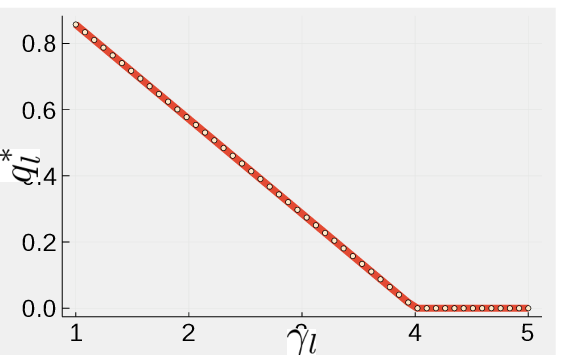}
    \end{subfigure}
    \hfill
    \begin{subfigure}[b]{0.49\columnwidth}
        \includegraphics[width=\linewidth]{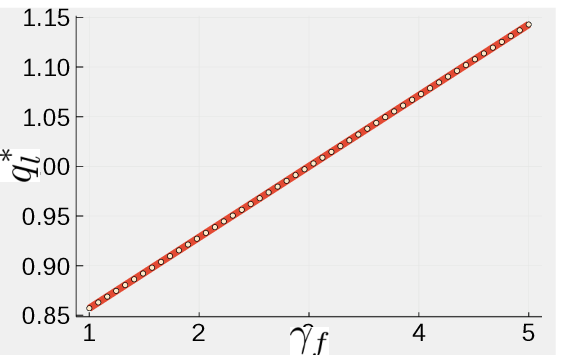}
    \end{subfigure}
    \caption{Comparison of the closed-form solution and the solution returned using our solver for a simple Stackelberg game with various parameter values. Even though this particular problem is an easy problem to solve without sophisticated machinery (a requirement to have a close-form solution available), our approach accurately recovers the optimal solution which is a useful validation of our algorithm.}
    \label{fig:closedform}
\end{figure}

\section{Validity of the Algorithm on Representative Robotics Applications} \label{subsec:robust_result}

We now demonstrate how our algorithm can find valid solutions to bilevel optimization problems corresponding to representative robotics applications. We first combine our differentiable solver with a state-of-the-art nonlinear optimization solver to recover robust trajectories in the framework of trajectory optimization. Next, we show the correctness, but also scalability, of our algorithm by applying it to the parameter estimation problem of a system with non-analytical dynamics (hard contact with friction).

\subsection{Robust Control}
When designing controllers or trajectories for robotic systems, taking potential external disturbances into account is essential. This challenge is often addressed in the framework of robust control. There are many ways to design robust controllers. Here we show how our solver allows us to design controllers that are robust to worst-case disturbances in a straightforward manner by casting the problem as a bilevel optimization problem. Notably, our approach avoids introducing additional decision variables to represent the noise which is often unavoidable with alternative optimization-based approaches to robust control such as \cite{LorenzettiLandryEtAl2019}.

More specifically, we look at the problem of designing a robust trajectory in state and input space using a direct method. Usually, trajectory optimization is performed by solving a nonlinear optimization problem in a form similar to
\begin{equation}
\begin{aligned}
& \underset{x_i,u_i;i=1\;\ldots m}{\text{minimize}}
& & \sum_{i=1}^{m}{J(x_i,u_i)} \\
& \text{subject to}
& & d(x_i,u_i,x_{i+1},u_{i+1}) = 0, i = 1, \ldots, m-1, \\
&
& & u_{\min} \leq u_i \leq u_{\max}, i = 1, \ldots, m-1,
\end{aligned}
\end{equation}
where $x_i$ is the state of the system at sample $i$ along the trajectory, $u_i$ the corresponding control input, $J$ some additive cost we are interested in minimizing along a trajectory, and $m$ is the total number of samples along the trajectory. The function $d$ captures the dynamic constraints between two consecutive time samples. This constraint depends on the chosen integration scheme. In the case of backward Euler for example, it would have the form
\begin{equation}
d(x_i,u_i,x_{i+1},u_{i+1}) := x_{i+1} - (x_i + h f(x_{i+1},u_{i+1})),
\end{equation}
where $h$ is the time between sample points $i$ and $i+1$. An in-depth exposition of direct trajectory optimization is beyond the scope of this paper and we refer the readers to \cite{Betts1998} for more information.

In order to design a trajectory that is more robust, we propose defining noise affecting the robot as the solution to a "worst-case" optimization problem. That worst-case noise can then easily be integrated in the trajectory optimization problem either as part of the objective or the constraints of the original problem giving variations of the bilevel problem
\begin{equation}
\begin{aligned}
& \underset{x_i,u_i;i=1\;\ldots m}{\text{minimize}}
& & \sum_{i=1}^{m}{J(x_i,u_i,\Psi(x_i,u_i))} \\
& \text{subject to}
& & d(x_i,u_i,x_{i+1},u_{i+1},\Psi(x_{i+1},u_{i+1})) = 0, \\
&
& & \hspace{40mm} i = 1, \ldots, m-1, \\
&
& & g(\Psi(x_i,u_i)) \leq 0, i = 1, \ldots, m-1,\\
&
& & u_{\min} \leq u_i \leq u_{\max}, i = 1, \ldots, m-1,
\end{aligned}
\end{equation}
where $\Psi(x,u)$ is the solution of the lower problem that computes the worst-case noise disturbance. In the general case, this lower problem can take many forms. We offer a specific example below.

For our specific example, we applied this technique to the trajectory optimization problem for a 7 degrees-of-freedom arm with a large flat plate on its end effector. In this example, the upper optimization problem contains the arm's dynamics, the start and end configuration constraints and the objective is to minimize the arm's velocity along the trajectory. Figure \ref{fig:nonrobustimg} shows the resulting trajectory when nothing with respect to potential disturbances (i.e., $\Psi(x,u) = 0$) is added to this trajectory optimization.

\begin{figure*}[ht]
    \centering
    \begin{subfigure}[t]{.195\textwidth}
        \includegraphics[trim={11cm 5cm 5cm 5cm},clip,width=\textwidth]{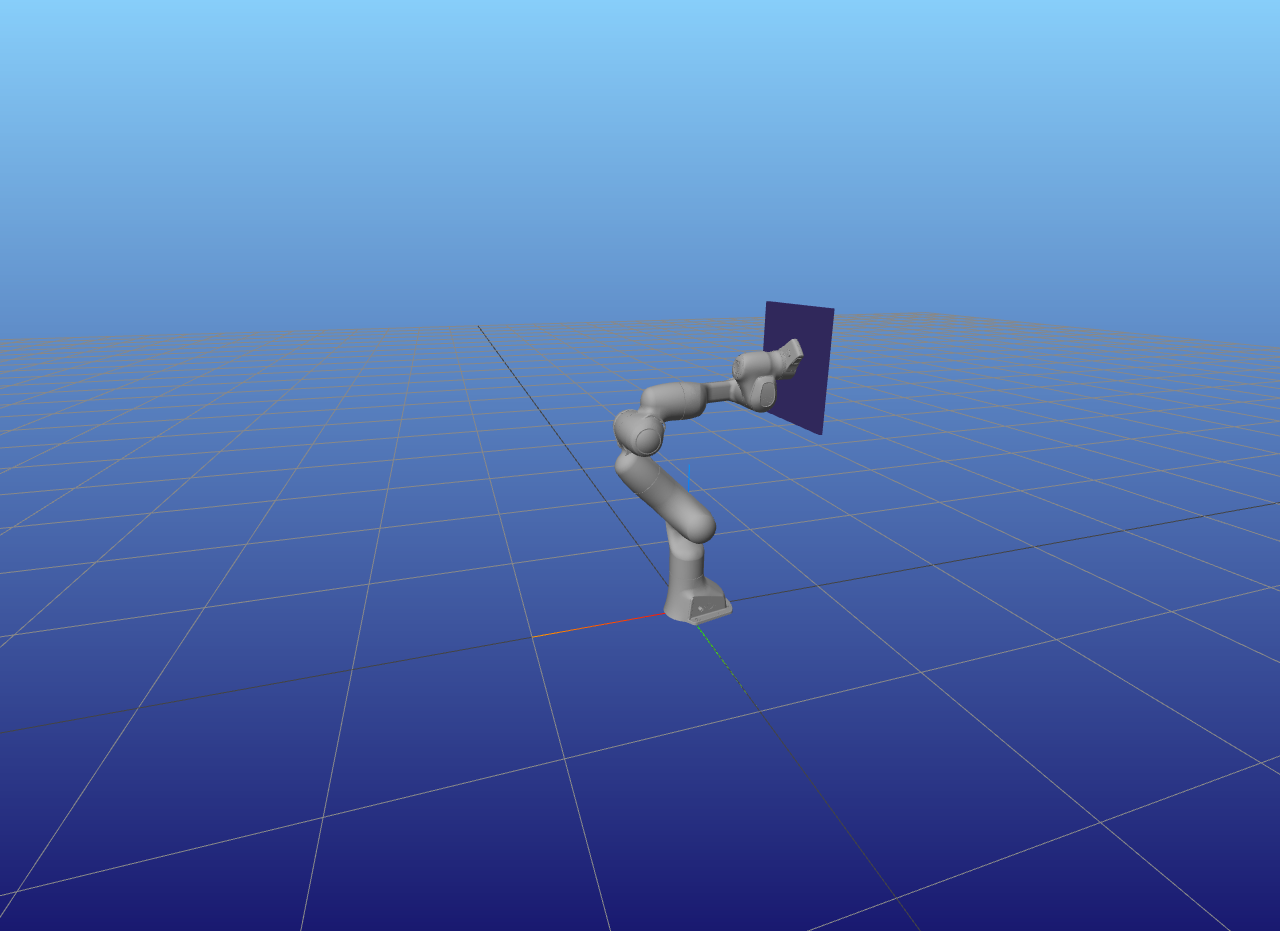}
    \end{subfigure}
    \begin{subfigure}[t]{.195\textwidth}
        \includegraphics[trim={11cm 5cm 5cm 5cm},clip,width=\textwidth]{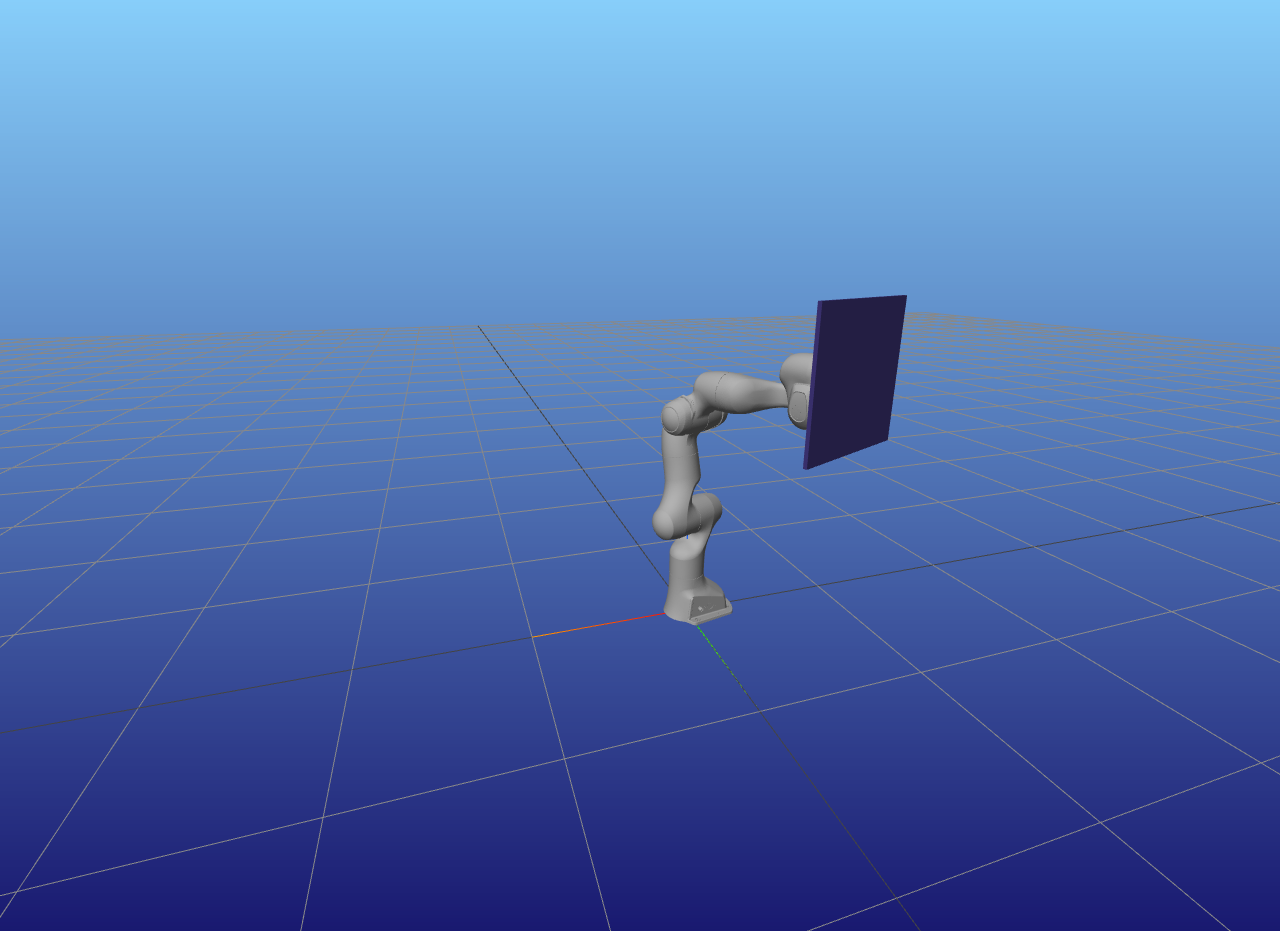}
    \end{subfigure}
    \begin{subfigure}[t]{.195\textwidth}
        \includegraphics[trim={11cm 5cm 5cm 5cm},clip,width=\textwidth]{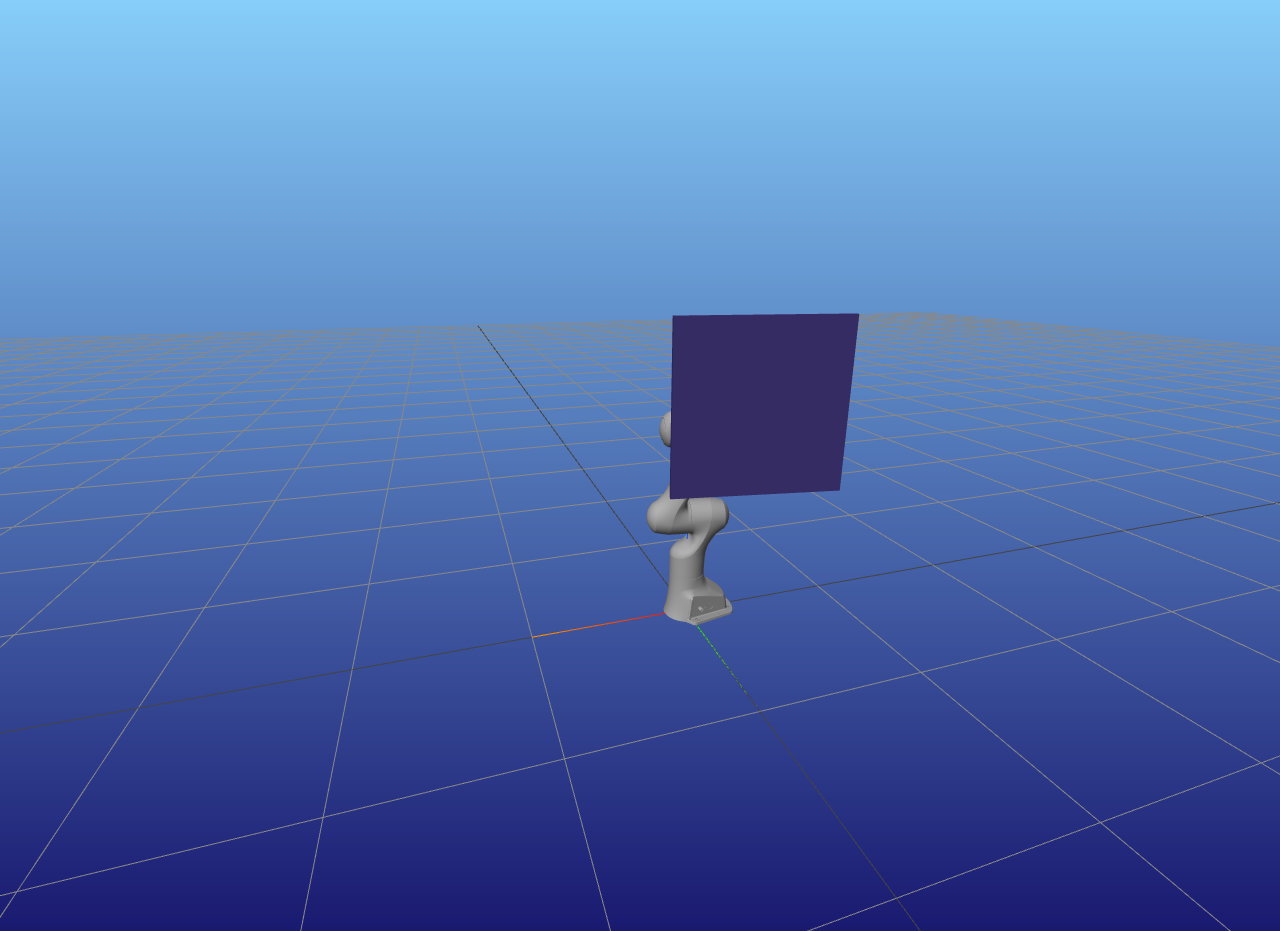}
    \end{subfigure}
    \begin{subfigure}[t]{.195\textwidth}
        \includegraphics[trim={11cm 5cm 5cm 5cm},clip,width=\textwidth]{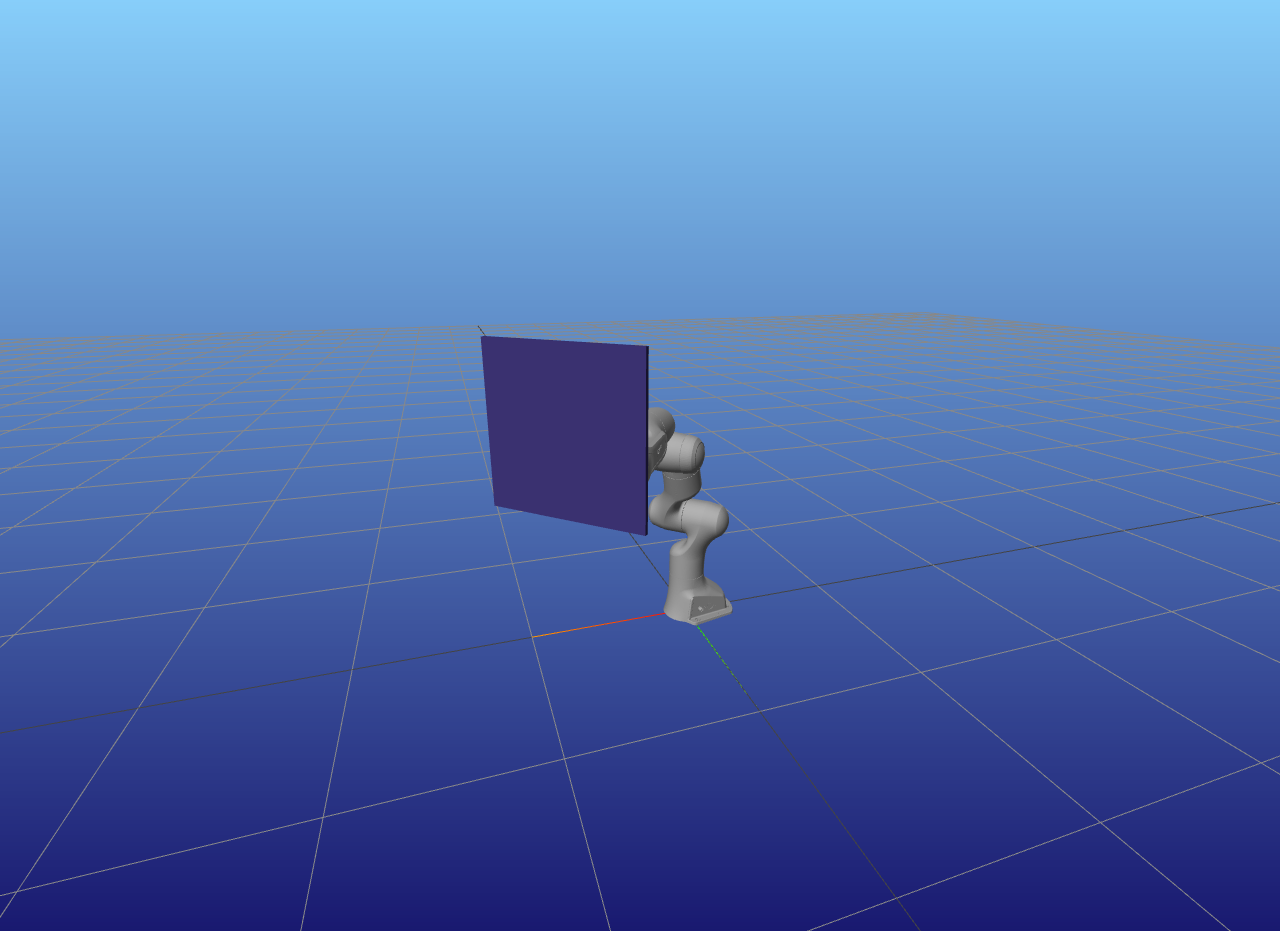}
    \end{subfigure}
    \begin{subfigure}[t]{.195\textwidth}
        \includegraphics[trim={11cm 5cm 5cm 5cm},clip,width=\textwidth]{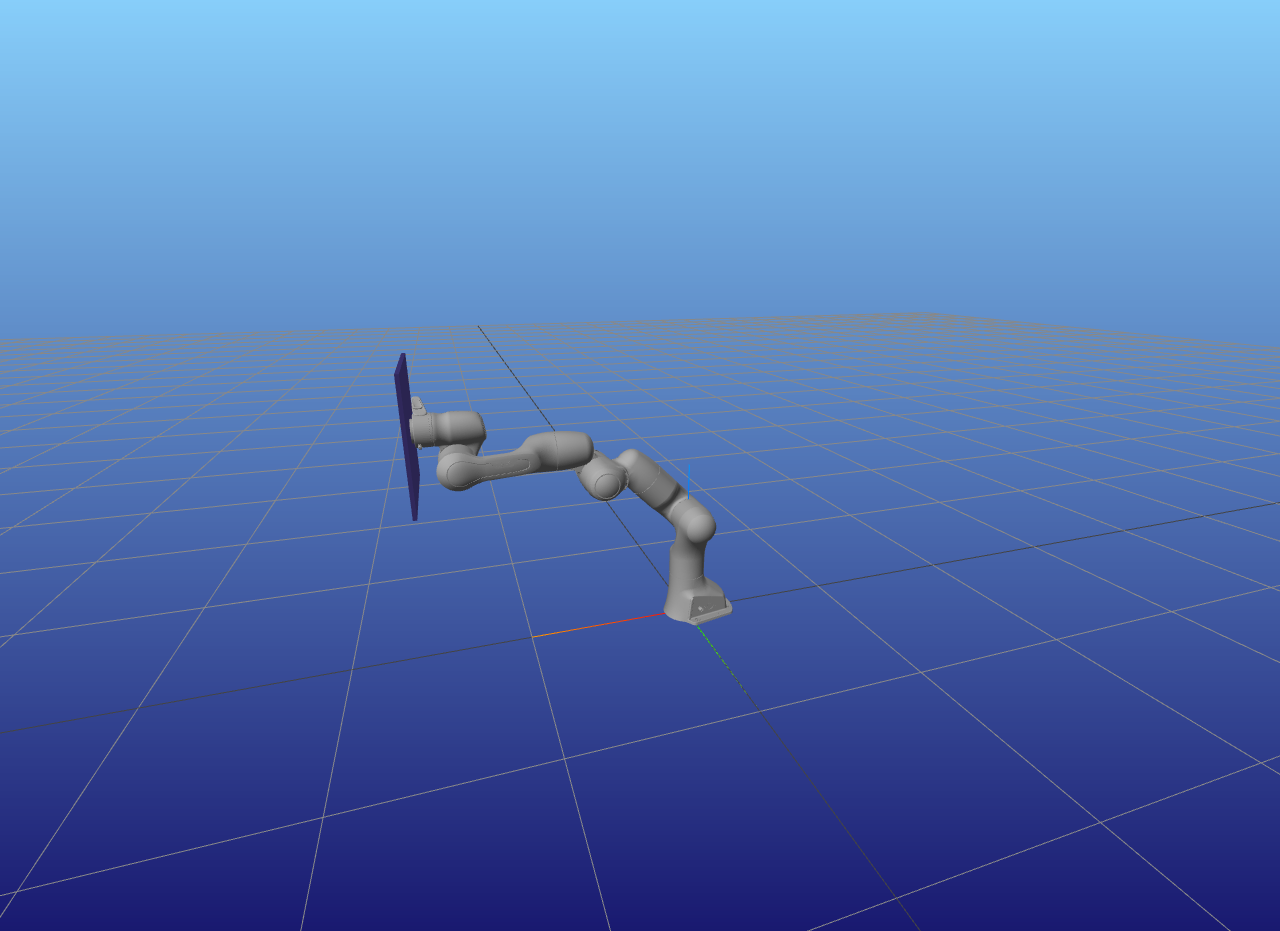}
    \end{subfigure}
    \caption{Trajectory that only minimizes joint velocities, with no regards to potential external disturbances.}
    \label{fig:nonrobustimg}
\end{figure*}

\begin{figure*}[ht]
    \centering
    \begin{subfigure}[t]{.195\textwidth}
        \includegraphics[trim={11cm 5cm 5cm 5cm},clip,width=\textwidth]{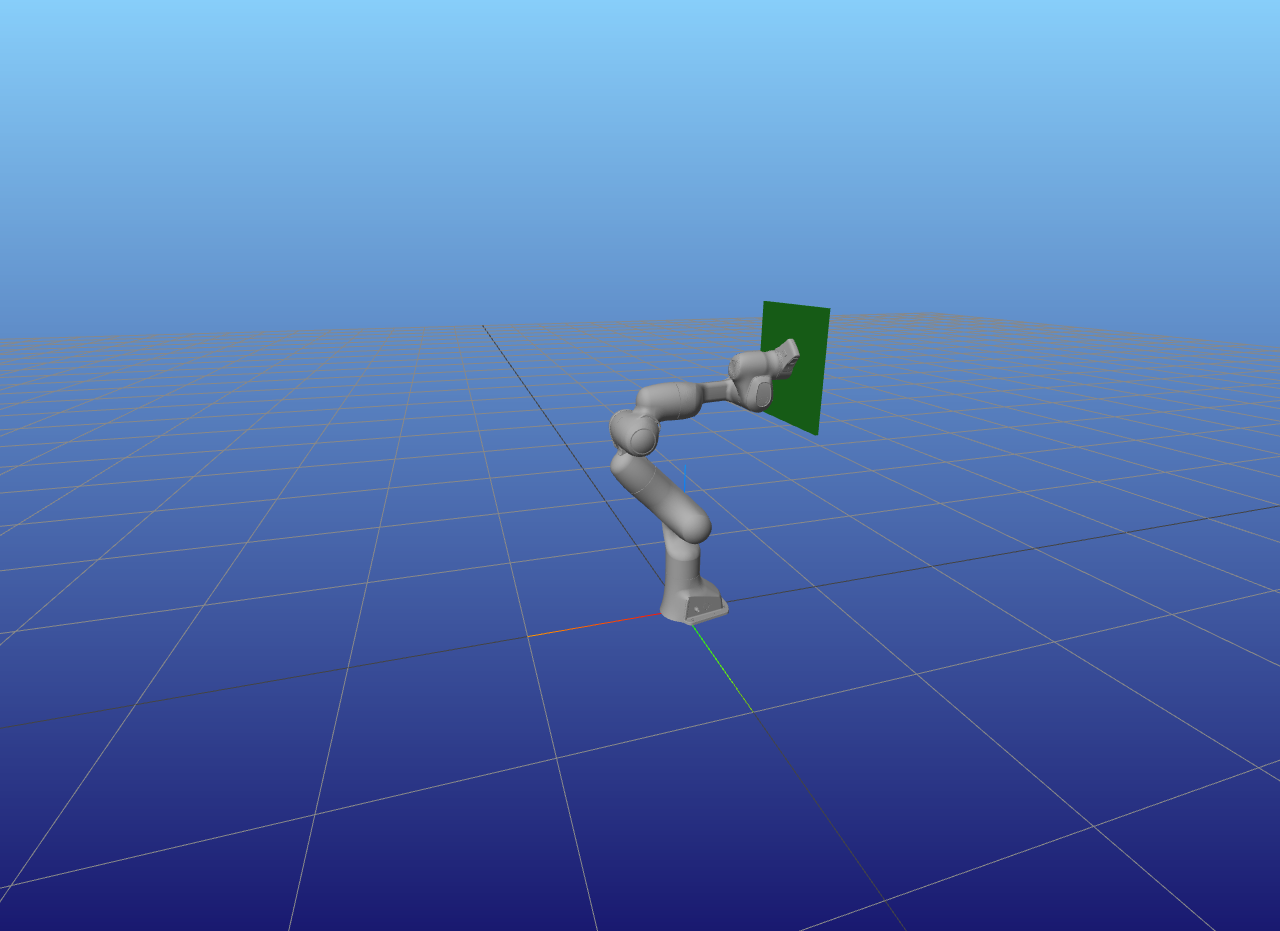}
    \end{subfigure}
    \begin{subfigure}[t]{.195\textwidth}
        \includegraphics[trim={11cm 5cm 5cm 5cm},clip,width=\textwidth]{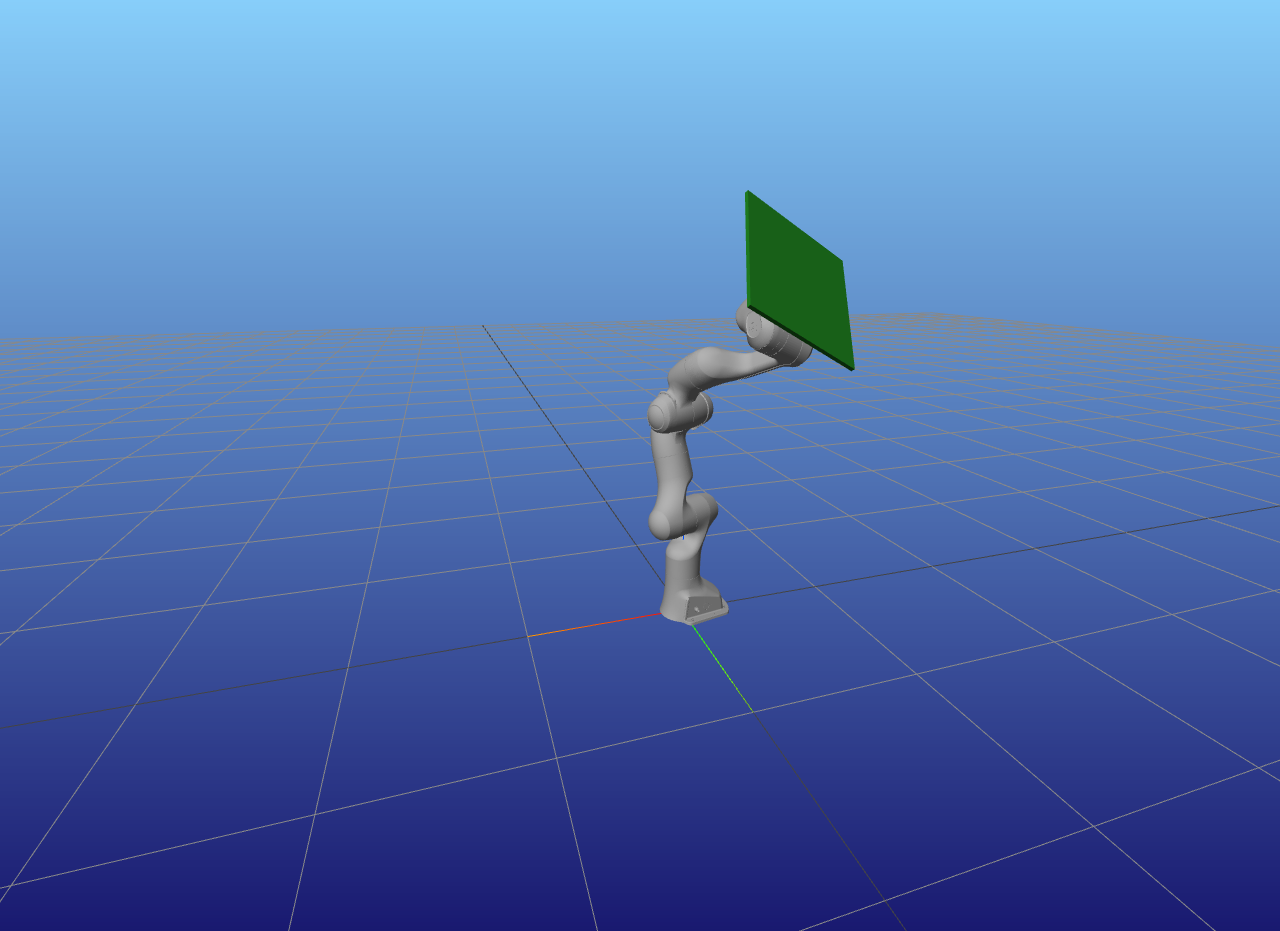}
    \end{subfigure}
    \begin{subfigure}[t]{.195\textwidth}
        \includegraphics[trim={11cm 5cm 5cm 5cm},clip,width=\textwidth]{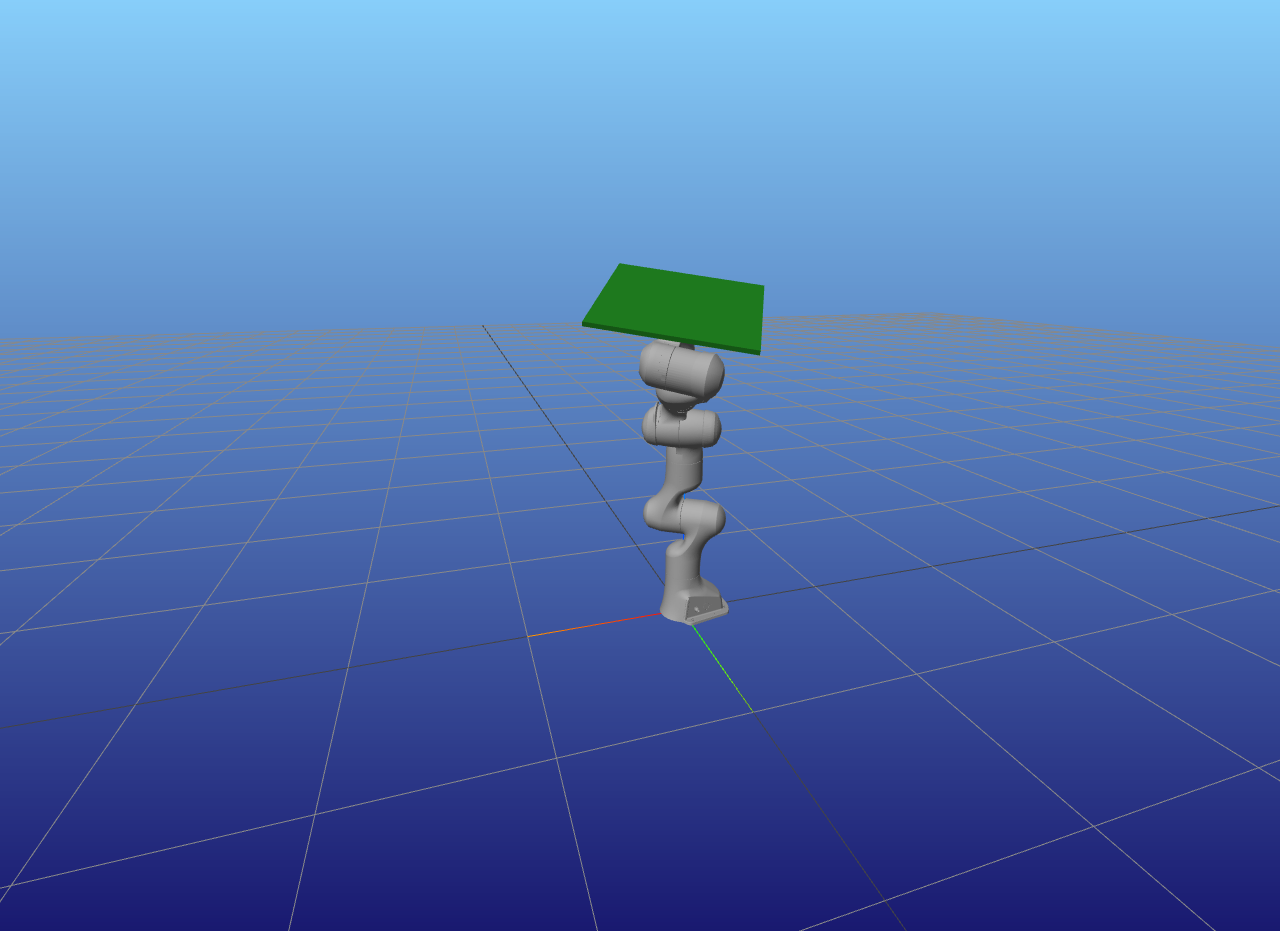}
    \end{subfigure}
    \begin{subfigure}[t]{.195\textwidth}
        \includegraphics[trim={11cm 5cm 5cm 5cm},clip,width=\textwidth]{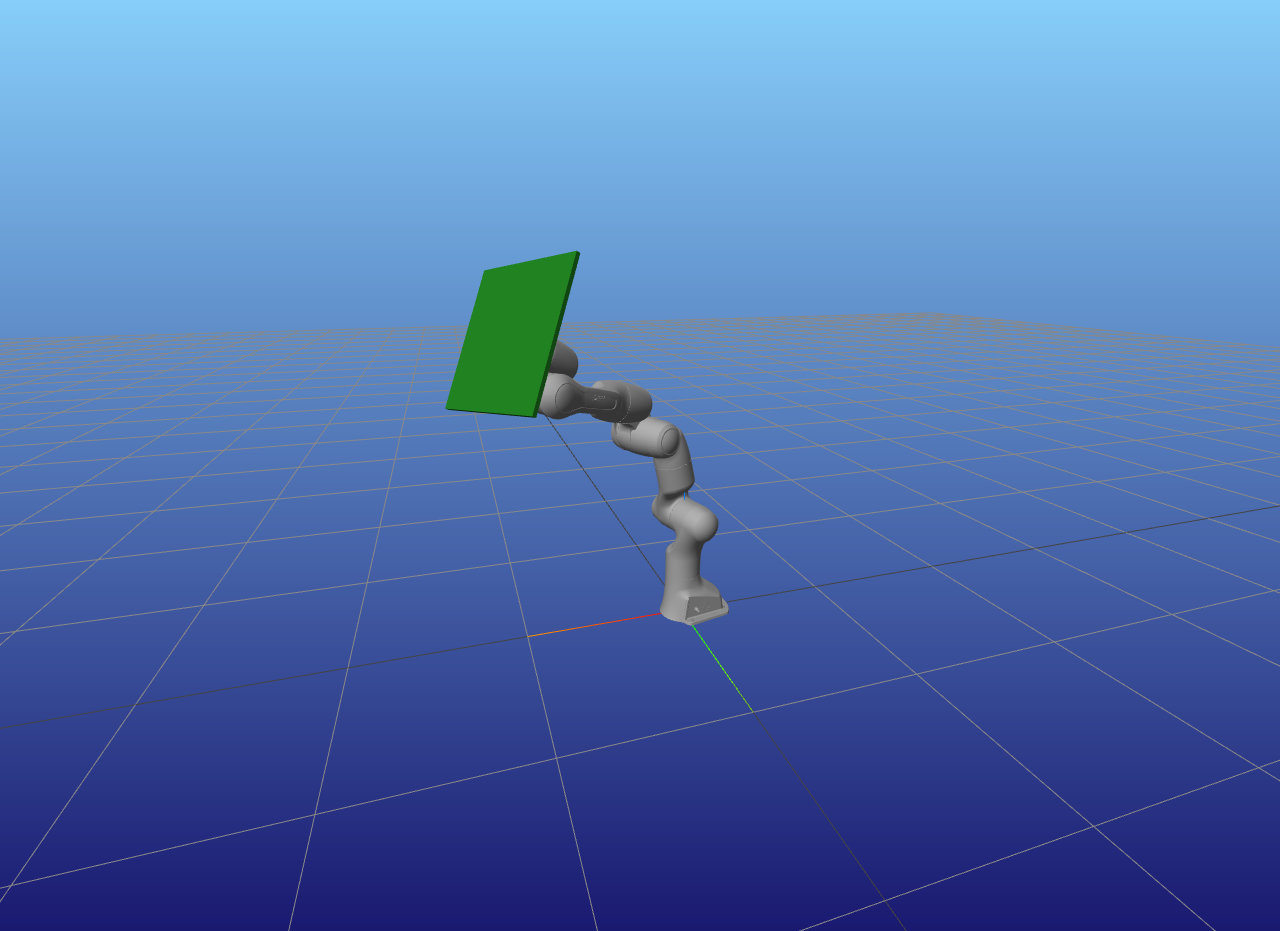}
    \end{subfigure}
    \begin{subfigure}[t]{.195\textwidth}
        \includegraphics[trim={11cm 5cm 5cm 5cm},clip,width=\textwidth]{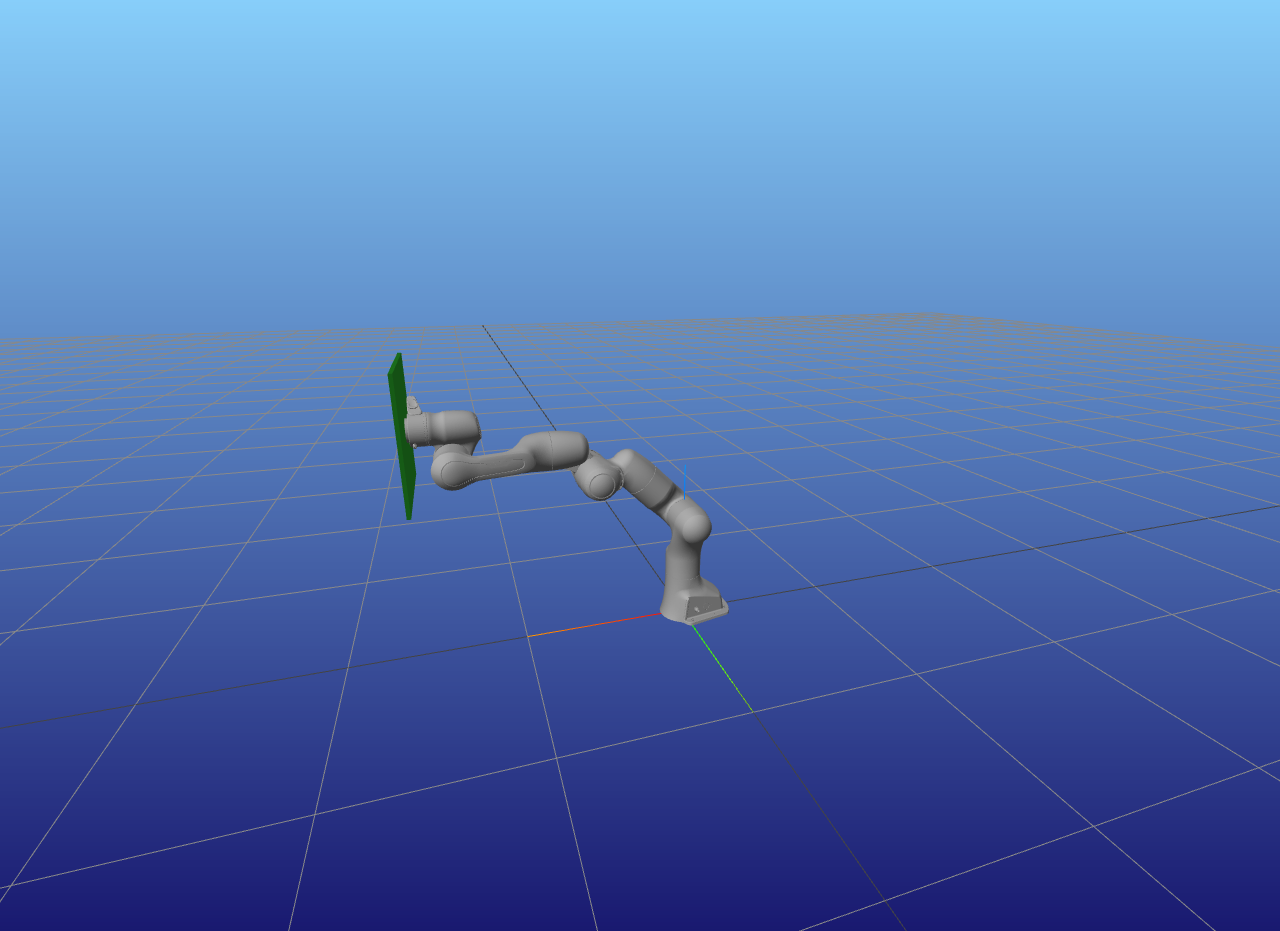}
    \end{subfigure}
    \caption{Trajectory that minimizes joint velocities and worst-case wind gust disturbance using bilevel optimization and our proposed solver. The wind gust has a smaller vertical bound than overall norm bound, and affects the arm proportionally to the effective area of the plate in the free stream.}
    \label{fig:robustimg}
\end{figure*}

Next we define a noise model for this example, namely a simple wind disturbance that has an overall $L_2$ norm constraint, but also has a smaller absolute value constraint on its vertical component (i.e. vertical wind gust strength is more limited than overall wind strength) making the noise model simple but not trivial. The wind affects the arm the most when it is perpendicular to the plate (which maximizes the effective area, and therefore drag, of the plate). Together, these define our worst-case noise model as the solution to the lower problem
\begin{equation}
\begin{aligned}
& \Psi(x,u) = 
& & \underset{w}{\text{argmin}}\{ P(x)^T w: \norm{w}_2 \leq 2, |w_z| \leq 1 \},\\
\end{aligned}
\end{equation}
where $P(x)$ computes the normal vector at the center of the flat plate given the current arm configuration $x$ and $w$ is the vector representing the wind strength and direction. Figure \ref{fig:robustimg} shows the resulting trajectory after including the bilevel noise model in the objective as 
\begin{equation}
J(x_i,u_i) = u_i^2 + 10 (P(x_i)^T \Psi(x_i,u_i))^2.
\end{equation}

Figure \ref{fig:robustobjplot} is then computed by taking the resulting trajectory and solving the lower problem to optimality with a state-of-the-art solver, SNOPT, as to avoid the possibility that our solver is returning inaccurate solutions of the lower problem. As can be seen in that plot, potential wind gust disturbance is effectively minimized throughout the trajectory, which is achieved by tilting the plate horizontally as it travels from the desired initial to final configuration.

\begin{figure}
    \centering
    \includegraphics[width=\columnwidth]{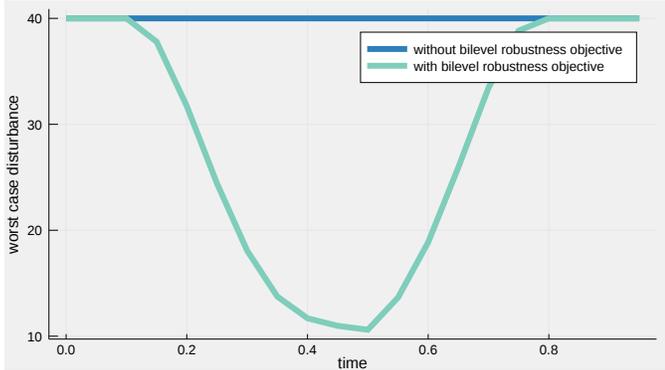}
    \caption{Comparison of the worst-case disturbance along a trajectory that was computed without taking disturbances into account in its objective, and one that did by exploiting the bilevel structure of the problem and our differentiable solver.}
    \label{fig:robustobjplot}
\end{figure}

Next, figure \ref{fig:robustconplot} shows the result of including the noise model as a constraint of the upper problem instead of as an objective
\begin{equation}
\Psi(x_i,u_i) \leq 25, i = 9,\ldots,11.
\end{equation}

Note that the constraint was only enforced on three samples in the middle of the trajectory to show the effect of the constraints more clearly. Once again, the resulting trajectory has the desired properties.

\begin{figure}
    \centering
    \includegraphics[width=\columnwidth]{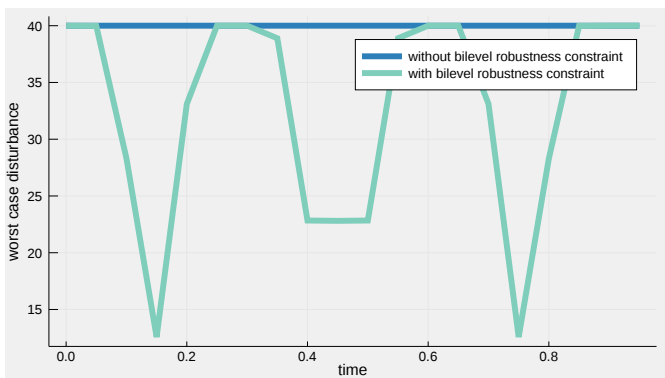}
    \caption{Comparison of the worst-case disturbance along a trajectory that was computed without taking disturbances into account as a constraint, and one that did for the middle three samples by exploiting the bilevel structure of the problem and our differentiable solver.}
    \label{fig:robustconplot}
\end{figure}

\subsection{Parameter Estimation}\label{subsec:param}

We now demonstrate the generality of our approach by applying it to the significantly different problem of parameter estimation. This application of our algorithm also provides an opportunity to demonstrate its scalability.

There are many ways to estimate parameters of a dynamical system. A common approach is to solve a mathematical program with the unknown parameters as decision variables and a dynamics residual evaluated over a observed sequences of states in the constraints or the objective of that problem. 

However, some systems, especially ones for which analytical dynamics are not always available, can bring additional challenges to this parameter estimation approach. For example, dynamics involving hard contact, unless approximations of the contact dynamics are made, are best expressed as solutions to optimization problems involving complementarity constraints.

Given a trajectory for a system that involves hard contact, formulating an optimization problem for parameter estimation requires not only the parameters as decision variables, but also the introduction of a set of decision variables for each contact and at each sample in the dataset. Clearly, this approach scales poorly, because the number of decision variables increases with each contact point and each sample added to the dataset. Scaling to the number of samples is also especially important in the presence of noisy measurements, because in such cases the qualities of the parameter estimates are proportional to the number of samples included.

Here, we demonstrate how our differentiable solver allows us to address this problem in a more scalable manner by casting it as a bilevel optmization problem and leveraging parallelism. The bilevel problem we propose to solve defines the contact dynamics in a lower problem, while solving the parameter estimation problem in the upper one. The key insight here is that even though we reduce the number of decision variables in the upper problem to match the number of parameters, we make the constraint evaluation significantly more expensive. However this can be compensated by parallelizing this constraint evaluation. This then allows our algorithm to regress parameters using more samples than the alternative approach.

First, we simulate the resulting trajectory of a 7 degrees of freedom arm pushing on a box that slides across the floor. Figure \ref{fig:pushimg} shows the types of trajectories the interaction produces. The contact points simulated are the four corners of the box and the arm's end effector. The contact dynamics are the widely popular ones introduced in \cite{StewartTrinkle2000}, although we do not linearize the dynamics. The rigid body dynamics are computed using the RigidBodyDynamics.jl package \cite{Koolencontributors2016}. The resulting trajectories are the sample points we use for parameter estimation. Defining $\mathcal{M}$ as the well-known manipulator equation
\begin{equation}
\begin{aligned}
    & \mathcal{M}(q_0,v_0,u,q,v,\Lambda) \\
    & \hspace{5mm} = H(q_0)(v - v_0) + C(q,v)  \\ 
    & \hspace{10mm} + G(q) - B(q,u) - P(q,v,\Lambda),
\end{aligned}
\end{equation}
where $H$ is the inertial matrix, $C$ captures Coriolis forces, $G$ captures potentials such as gravity, $B$ maps control inputs to generalized forces and $P$ maps external contact forces to generalized forces and $\Lambda$ are the external contact forces. We then formulate the following lower optimization problem
\begin{figure*}[ht]
    \centering
    \begin{subfigure}[t]{.195\textwidth}
        \includegraphics[trim={8cm 5cm 11cm 5cm},clip,width=\textwidth]{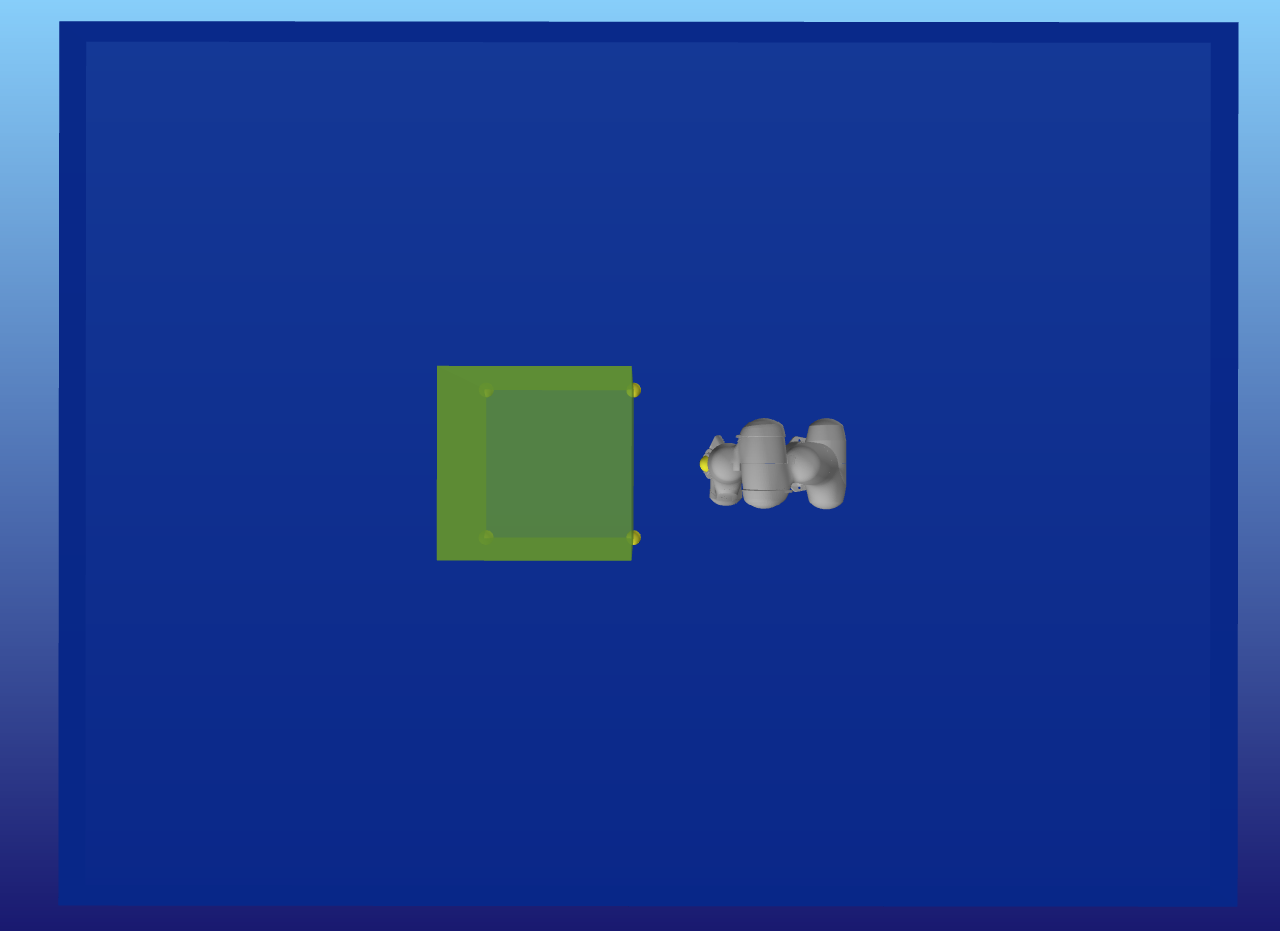}
    \end{subfigure}
    \begin{subfigure}[t]{.195\textwidth}
        \includegraphics[trim={8cm 5cm 11cm 5cm},clip,width=\textwidth]{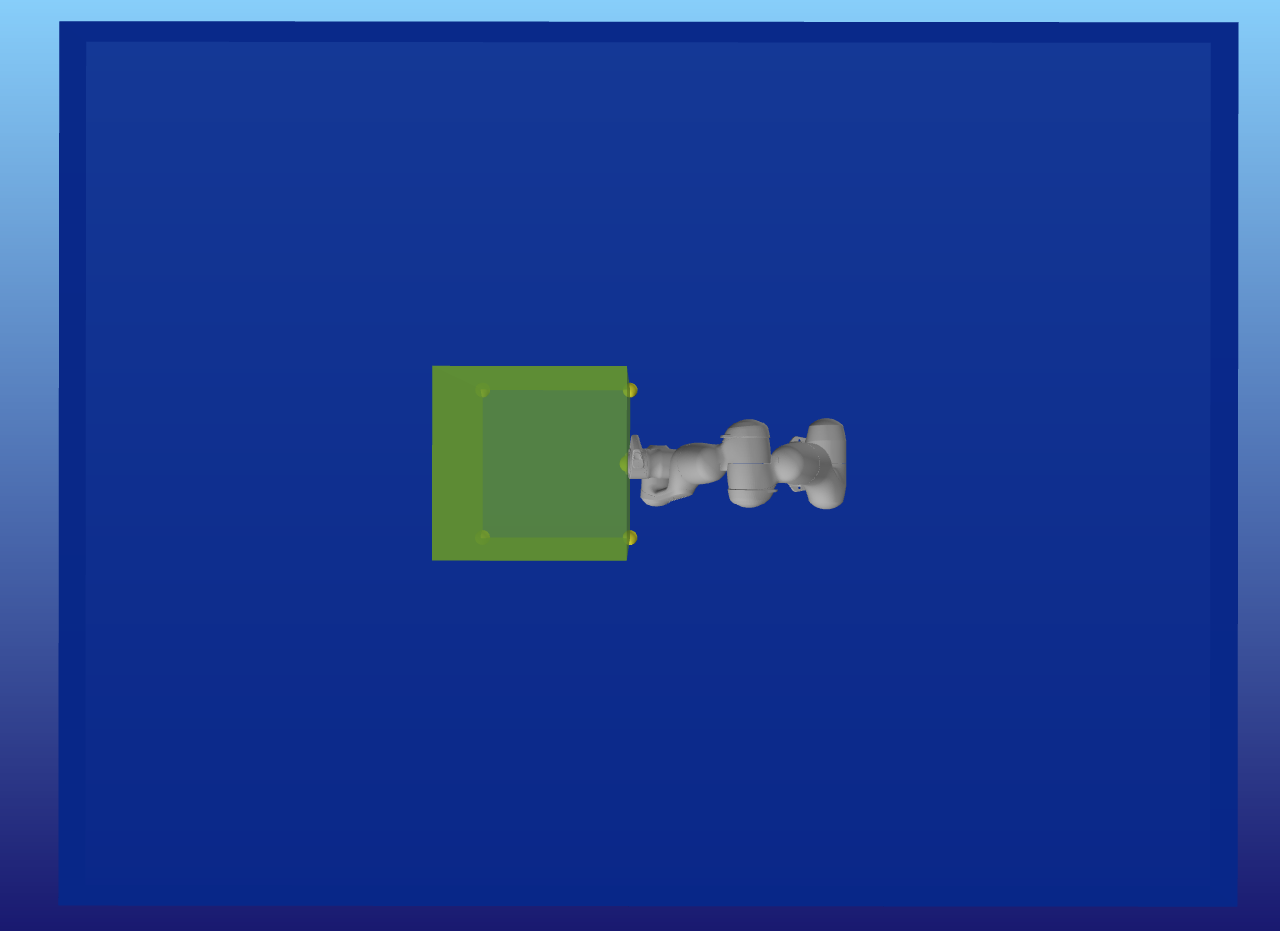}
    \end{subfigure}
    \begin{subfigure}[t]{.195\textwidth}
        \includegraphics[trim={8cm 5cm 11cm 5cm},clip,width=\textwidth]{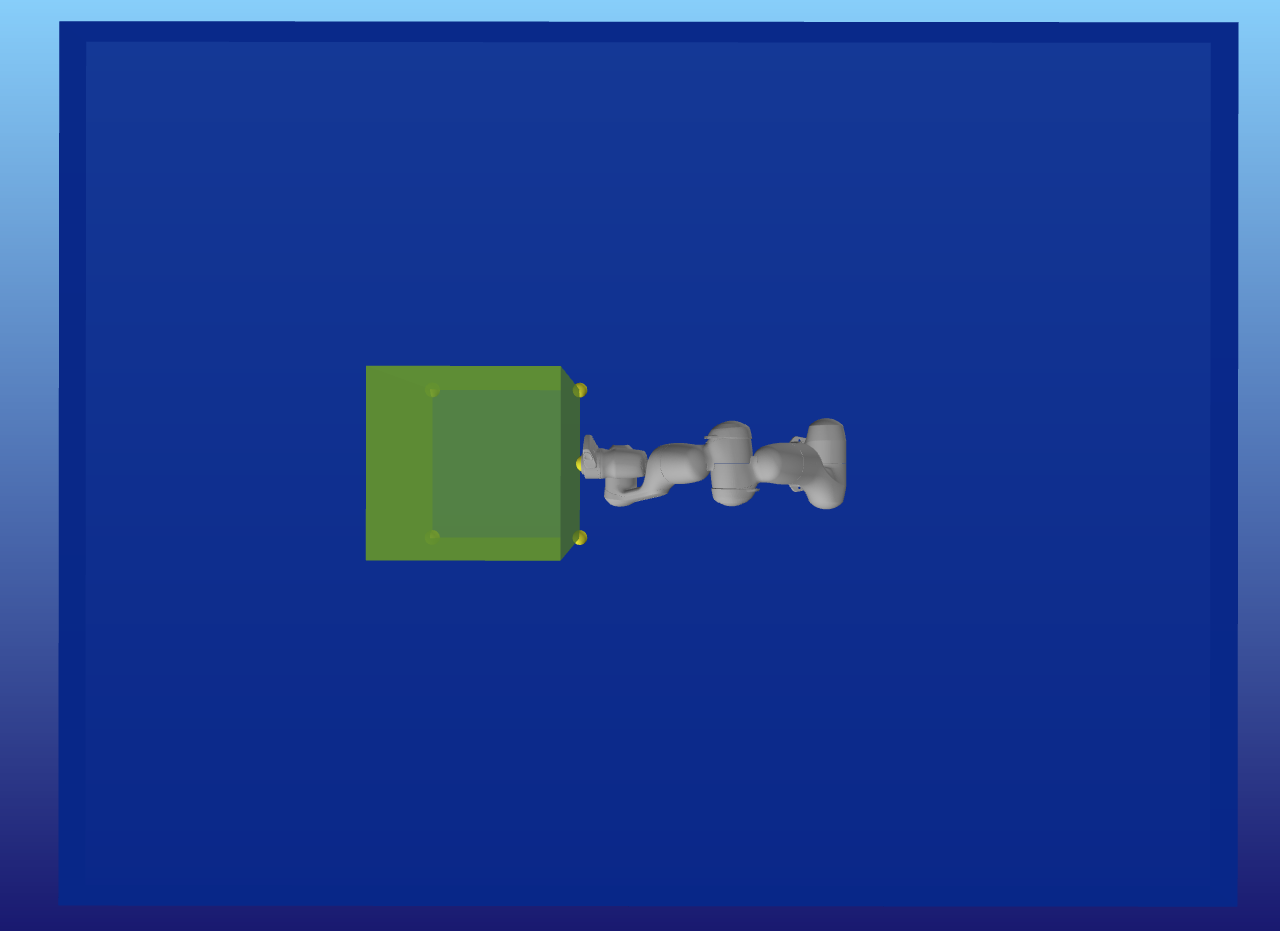}
    \end{subfigure}
    \begin{subfigure}[t]{.195\textwidth}
        \includegraphics[trim={8cm 5cm 11cm 5cm},clip,width=\textwidth]{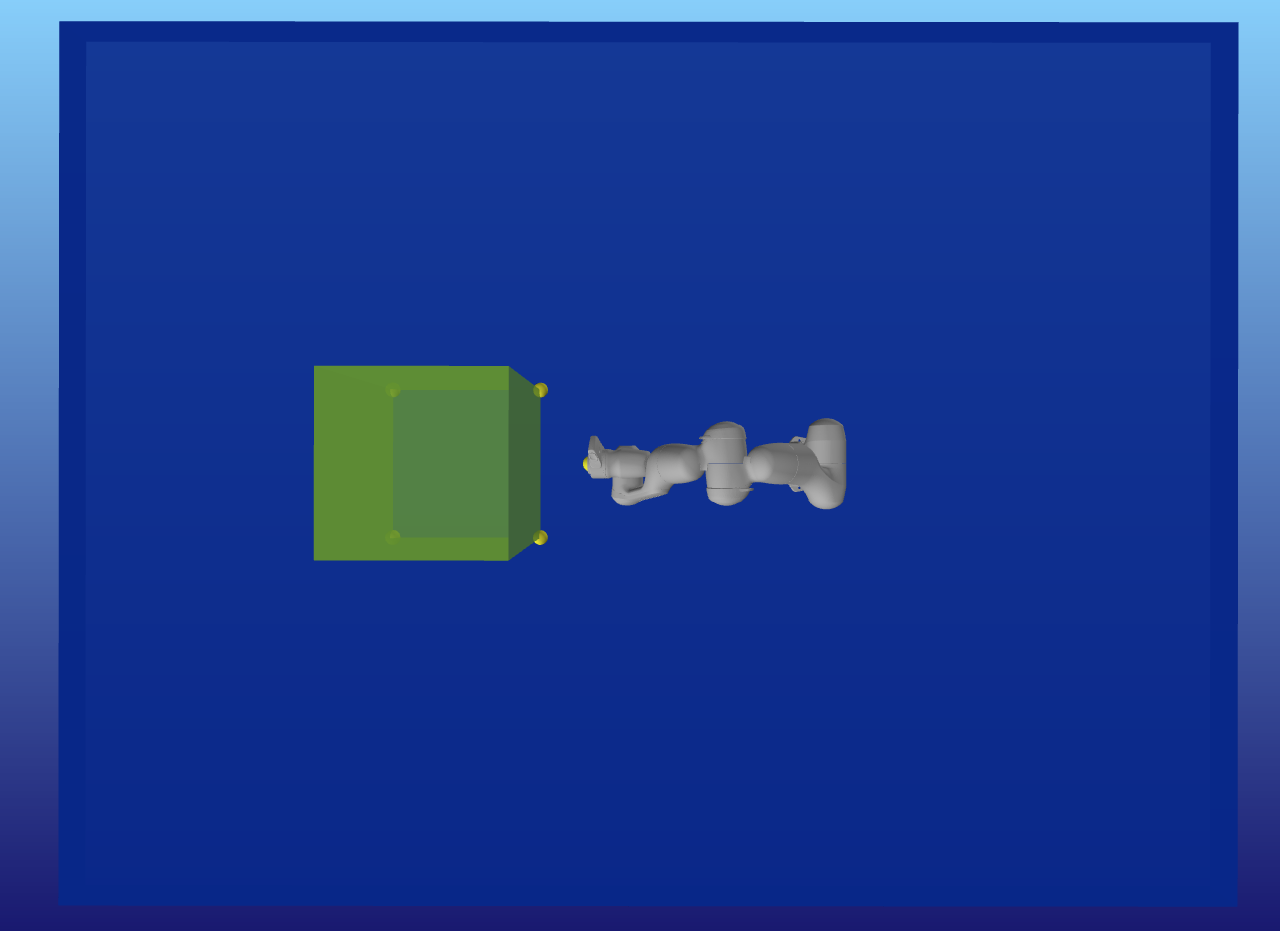}
    \end{subfigure}
    \begin{subfigure}[t]{.195\textwidth}
        \includegraphics[trim={8cm 5cm 11cm 5cm},clip,width=\textwidth]{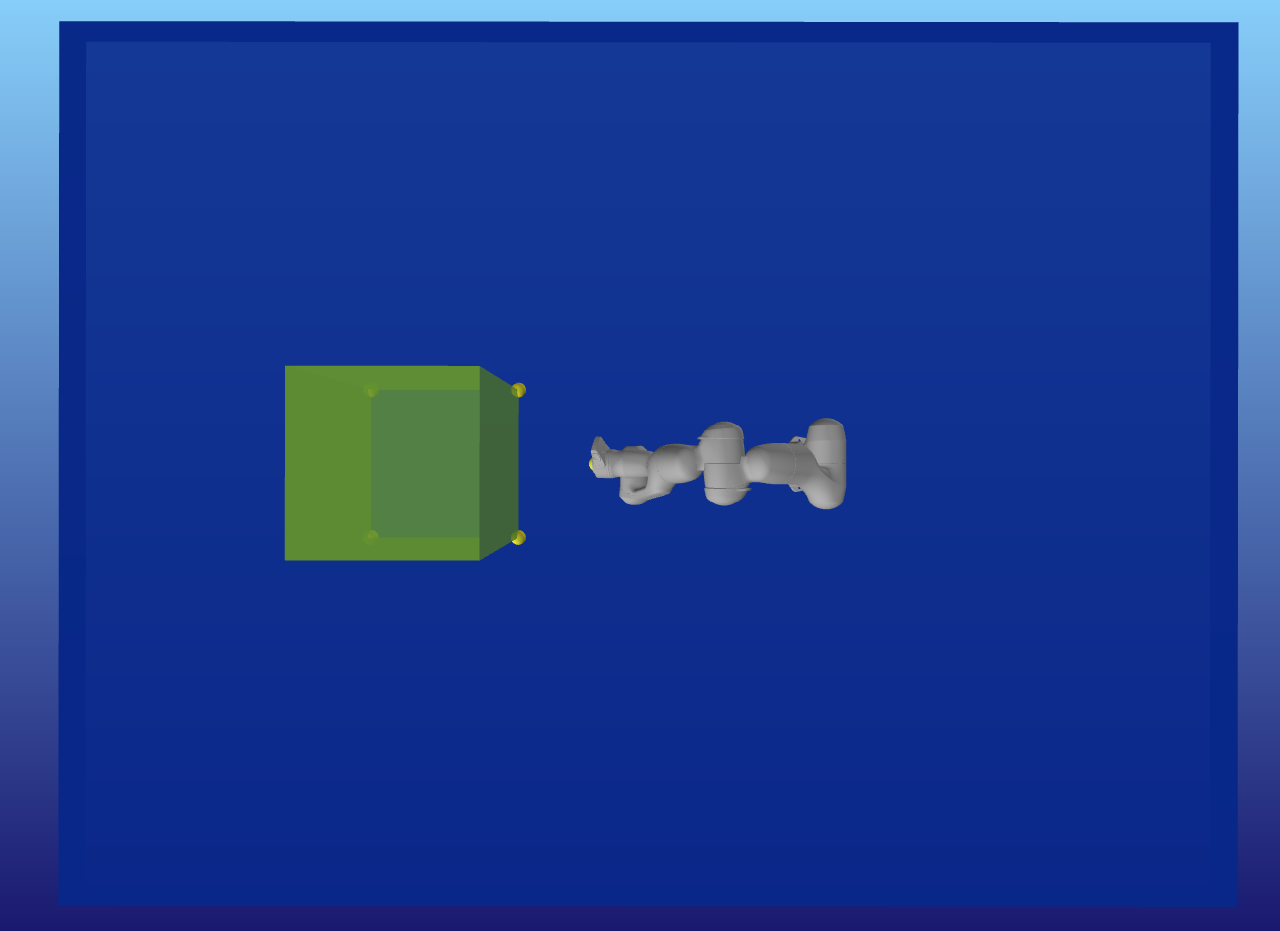}
    \end{subfigure}
    \caption{Simulation of a 7 degrees of freedom robotic arm pushing a box that slides across the floor to estimate the friction coefficients.}
    \label{fig:pushimg}
\end{figure*}
\begin{equation}
\begin{aligned}
& \Psi(q_0,v_0,u,q,v) = 
& & \underset{\Lambda = \{\beta,\lambda,c_n\}}{\text{argmin}}\{ \norm{\beta}^2 + \norm{c_n}^2: \\
& 
& & \mathcal{M}(q_0,v_0,u,q,v,\Lambda) = 0, \\
&
& & \lambda \mathbf{e} + D(q)^Tv \geq 0, \\
&
& & \mu c_n - \mathbf{1}^T\beta \geq 0, \\
&
& & \phi(q) c_n = 0, \\
& 
& & (\lambda \mathbf{e} + D(q)^T v)^T \beta = 0,\\
&
& & (\mu c_n - \mathbf{1}^T \beta) \lambda = 0,\\
&
& & \beta, c_n, \lambda \geq 0 \},
\end{aligned}
\end{equation}
where the components of $\Lambda$, namely $\beta$, $c_n$ and $\lambda$ are variables describing the friction and normal forces and $D$ computes the contact friction basis. We refer the readers to \cite{StewartTrinkle2000} for a more complete description of the contact dynamics used here. We then solve the following upper optimization problem
\begin{equation}
\begin{aligned}
& \underset{\mu}{\text{minimize}}
& & \sum_{i=1}^{N}{\mathcal{M}(q_i,v_i,u_{i+1},q_{i+1},v_{i+1},\Psi(.))_2^2} \\
& \text{subject to}
& & 0 \leq \mu \leq 1, \\
\end{aligned}
\end{equation}
with $(q_i,v_i,u_{i+1},q_{i+1},v_{i+1})$ as inputs to $\Psi$ using instances of our solver running across multiple threads to handle the lower problem and SNOPT to take care of the upper one.

In our specific parallelization scheme, one thread corresponds to one sample in the dataset, but other parallelization schemes are also possible. Note that if the samples are noisy, the lower problem is not necessarily feasible. In such cases, our solver can be interpreted as minimizing the residual of the dynamics. Also note that since we are not using the KKT conditions to retrieve the gradient of the lower problem, not solving it to optimality or even feasibility still results in well-defined gradients.

Table \ref{tab:estimation} contains the results of running the estimation approaches for various numbers of samples and different friction coefficients. The important take-away from these results is even though our current implementation of the solver is slower than doing the parameter estimation classically for small problems, our approach scales better with the number of samples thanks to its straighforward parallelization. Conceptually, this is true for the same reasons that it is simple to compute losses over samples in parallel when doing gradient descent in machine learning contexts, except in this case the ``loss" is itself a complex nonlinear optimization.

\newcolumntype{g}{>{\columncolor{Gray}}c}
\begin{table}[]
\begin{tabular}{|g|c|c|c|c|c|c|}
\hline
\multicolumn{2}{|c|}{}  & \multicolumn{2}{|>{\cellcolor{Green}}c|}{\color{GrayPlot}\textbf{Bilevel}} & \multicolumn{3}{|>{\cellcolor{Blue}}c|}{\color{GrayPlot}\textbf{Classical}}\\
\hline
\multicolumn{1}{|c|}{\textbf{$\mu$}} & \multicolumn{1}{|c|}{\textbf{\# samples}} &
\multicolumn{1}{|p{9mm}|}{\textbf{time (s)}} & 
\multicolumn{1}{|p{8mm}|}{\textbf{scaled}} &
\multicolumn{1}{|p{9mm}|}{\textbf{time (s)}} & 
\multicolumn{1}{|p{8mm}|}{\textbf{scaled}} & \multicolumn{1}{|p{7mm}|}{\textbf{\# var}} \\
\hline
.23 & 6 & 4.0 & 1.0 & .10 & 1.0 & 217 \\
.19 & 11 & 7.5 & 1.9 & .32 & 3.1 & 397 \\
.19 & 19 & 12.6 & 3.2 & .96 & 9.6 & 685 \\
.23 & 21 & 14.2 & 3.6 & 1.3 & 13.0 & 757 \\
.19 & 23 & 16.4 & 4.1 & 1.4 & 14.0 & 829 \\
.15 & 34 & 23.5 & 5.9 & 3.5 & 35.0 & 1225 \\
.15 & 44 & 30.7 & 7.7 & 24.55 & 246.0 & 1585 \\
\hline
\end{tabular}
\caption{Results of estimating the friction coefficients on data produced by a simulation with hard contact. Our approach, labeled as "Bilevel", scales better with respect to the number of samples than the alternative approach, labeled as "Classical", that introduces decision variables for the unobserved contact forces in the dataset.}
\label{tab:estimation}
\end{table}

% .19 & 26 & 1 & 12.7 & \cellcolor{Red} x & \cellcolor{Red} x & \cellcolor{Red} 937 \\
% .15 & 34 & 1 & 24.3 & \cellcolor{Red} x & \cellcolor{Red} x & \cellcolor{Red} 1225 \\

\section{Conclusion} 
\label{sec:conclusion}
In this work, we develop a differentiable general purpose nonlinear program solver based on augmented Lagrangian methods. We then apply our solver to a toy example and two important problems in robotics, robust control and parameter estimation, by casting them as bilevel optimization problem. In future work, we intend to improve the performance of our solver by leveraging more advanced optimization techniques such as warm starting and advanced parallelization schemes like GPU programming. We also plan to apply our solver to other problems in robotics that might exhibit a bilevel structure, such as imitation learning and trajectory optimization through contact.

\section*{Acknowledgments}
The authors of this work were partially supported by the Office of Naval Research, Sea-Based Aviation Program.

%% Use plainnat to work nicely with natbib. 

\bibliographystyle{plainnat}
\bibliography{biblio,ASL_papers}

\end{document}